\documentclass[a4paper,11pt]{article}
\usepackage[top=4cm, bottom=4cm, left=3cm, right=3cm]{geometry}

\usepackage[utf8]{inputenc} % allow utf-8 input
\usepackage[T1]{fontenc}    % use 8-bit T1 fonts
\usepackage{url}            % simple URL typesetting
\usepackage{amsfonts}       % blackboard math symbols
\usepackage{nicefrac}       % compact symbols for 1/2, etc.
\usepackage{microtype}      % microtypography

\usepackage{amssymb}
\usepackage{amsmath}
\usepackage{amsthm}

\usepackage{thmtools, thm-restate}
\usepackage{mathrsfs} %pour mathscr
\usepackage[usenames,dvipsnames]{pstricks}
\usepackage{euler}
\usepackage{euscript}
\usepackage{graphicx}
\usepackage{charter}
\usepackage{mdframed}
\usepackage{enumerate, subfigure, color}
\usepackage[colorlinks,hyperindex]{hyperref}
\usepackage{bibunits}

\usepackage{newfloat}

\usepackage{booktabs}
\usepackage{cite}

\usepackage{floatrow}

\usepackage[font=small,skip=0pt]{caption}

\usepackage{hyperref}       % hyperlinks
\usepackage{url}                % simple URL typesetting
\usepackage{nicefrac}         % compact symbols for 1/2, etc.
\usepackage{microtype}      % microtypography
\usepackage{amsmath,amssymb,amsthm} 
\usepackage{algorithm}
\usepackage{algpseudocode}
\usepackage{thmtools,thm-restate}
\usepackage[nameinlink,capitalize]{cleveref}
\hypersetup{colorlinks={true},linkcolor={blue},citecolor=blue}
\usepackage{todonotes}
\usepackage{multicol}
\usepackage[sort&compress]{natbib}
\usepackage{accents}

\newcommand\thickbar[1]{\accentset{\rule{.365em}{.5pt}}{#1}}
\newcommand{\X}{\mathcal{X}}
\newcommand{\Y}{\mathcal{Y}}
\newcommand{\Z}{\mathcal{Z}}

\newcommand{\trans}{^{\scriptscriptstyle \top}}
\newcommand{\R}{\mathbb{R}}
\renewcommand{\P}{{\rm P}}
\newcommand{\N}{\mathbb{N}}

\newcommand{\la}{\lambda}

\newcommand{\proj}{\ensuremath{\text{\rm proj}}}
\newcommand{\tr}{\ensuremath{\text{\rm Tr}}}

\newcommand{\ran}{\ensuremath{\text{\rm Ran}}}

\newcommand{\argmin}{\operatornamewithlimits{argmin}}

\newcommand{\ee}{{\mathcal{E}}}

\newcommand{\rx}{{R}}

\newcommand{\Zn}{Z}

\newcommand{\LL}{\mathcal{L}}

\newcommand{\cR}{\mathcal{R}}
\newcommand{\T}{\mathcal{T}}
\newcommand{\Real}{\mathbb{R}}
\newcommand{\Exp}{\mathbb{E}}
\newcommand{\psd}{\mathbb{S}_+^d}

\newcommand{\vertiii}[1]{{\left\vert\kern-0.25ex\left\vert\kern-0.25ex\left\vert #1 
    \right\vert\kern-0.25ex\right\vert\kern-0.25ex\right\vert}}

\declaretheorem[name=Theorem,refname=Thm.]{theorem}

\declaretheorem[name=Lemma,sibling=theorem]{lemma}
\declaretheorem[name=Proposition,refname=Prop.,sibling=theorem]{proposition}
\declaretheorem[name=Remark]{remark}
\declaretheorem[name=Corollary,refname=Cor.,sibling=theorem]{corollary}

\declaretheorem[name=Assumption,refname=Asm.]{assumption}

\newcommand{\task}{\mu}  % task over \X \times \Y
\newcommand{\Ss}{\mathcal{S}}  % space of side information
\newcommand{\Tt}{\mathcal{M}}   % space of tasks over \X \times \Y
\newcommand{\D}{\mathcal{D}}  % space of datasets
\newcommand{\env}{\rho}  % distribution over \P \times \Ss
\newcommand{\ms}{{\env_\Ss}}  % marginal side information
\newcommand{\mt}{{\env_\Tt}}  % marginal tasks
\crefname{assumption}{Asm.}{Asm.}
\crefname{equation}{Eq.}{Eq.}
\crefname{figure}{Fig.}{Fig.}
\crefname{table}{Tab.}{Tab.}
\crefname{section}{Sec.}{Sec.}
\crefname{theorem}{Thm.}{Thm.}
\crefname{proposition}{Prop.}{Prop.}
\crefname{fact}{Fact}{Facts}
\crefname{lemma}{Lemma}{Lemmas}
\crefname{corollary}{Cor.}{Cor.}
\crefname{example}{Ex.}{Ex.}
\crefname{remark}{Rem.}{Rem.}
\crefname{algorithm}{Alg.}{Algorithms}
\crefname{appendix}{App.}{App.}
\crefname{algorithm}{Alg.}{Alg.}

\providecommand{\scal}[2]{\left\langle{#1},{#2}\right\rangle}

\providecommand{\nor}[1]{\left\|{#1}\right\|}
\renewcommand{\paragraph}[1]{~\newline{\bfseries #1.}}
\renewcommand{\P}{{\mathcal{P}}}

\title{\sffamily\huge\bf Conditional Meta-Learning of Linear Representations}

\date{}

\author{\hspace{-.3truecm}Giulia Denevi$^{1}$,
Massimiliano Pontil$^{1,2}$ and Carlo Ciliberto$^{1,2}$ \\
\footnotesize $^1$University College of London (UK), $^2$Istituto Italiano di Tecnologia (Italy) \\
{\footnotesize \em g.denevi@ucl.ac.uk, massimiliano.pontil@iit.it, c.ciliberto@ucl.ac.uk}
}

\begin{document}

\maketitle

\begin{abstract}
  \noindent Standard meta-learning for representation learning aims to find a common representation to be shared across multiple tasks. The effectiveness of these methods is often limited when the nuances of the tasks’ distribution cannot be captured by a single representation. In this work we overcome this issue by inferring a conditioning function, mapping the tasks' side information (such as the tasks' training dataset itself) into a representation tailored to the task at hand. We study environments in which our conditional strategy outperforms standard meta-learning, such as those in which tasks can be organized in separate clusters according to the representation they share. We then propose a meta-algorithm capable of leveraging this advantage in practice. In the unconditional setting, our method yields a new estimator enjoying faster learning rates and requiring less hyper-parameters to tune than current state-of-the-art methods. Our results are supported by preliminary experiments.
\end{abstract}

\footnotetext[1]{Department of Computer Science, University College London, London, United Kingdom}
\footnotetext[2]{Computational Statistics and Machine Learning, Istituto Italiano di Tecnologia, Genova, Italy}

%--------------------------------------------------------------------------------------------------------------------------------------------

\section{Introduction}
\label{Introduction}

Learning a shared representation among a class of machine learning problems is a well-established approach used both in multi-task learning 
\cite{argyriou2008convex,jacob2009clustered,caruana1997multitask}
and meta-learning \cite{finn2019online,denevi2019online,balcan2019provable,finn2018meta,tripuraneni2020provable,maurer2009transfer,pentina2014pac,hazan19,bertinetto2018meta}. 
The idea behind this methodology is to consider two nested problem: at the within-task level an empirical risk
minimization is performed on each task, using inputs transformed by the current representation, on the outer-task (meta-) level, such a representation is updated taking into account the errors of the within-task algorithm on previous tasks.

Such a technique was shown to be advantageous in contrast to solving each task independently when the tasks share a low dimensional representation, see e.g. \cite{maurer2016benefit,maurer2013sparse,denevi2019online,maurer2009transfer,tripuraneni2020provable,balcan2019provable,khodak2019adaptive,hazan19}. However, in real world applications we often deal with heterogeneous 
classes of learning tasks, which may overall be only loosely related. Consequently, the tasks’ commonalities may not be captured well by a single representation shared among all the tasks.
%, implying that the above approach will fail. 
This is for instance the case in which the tasks can be organized in different groups (clusters), where only tasks belonging to the same cluster share the same low-dimensional representation.
%MP: fare 2 esempi? user modelling? vision? use previous work
%in which the complexity of the environment can not be captured by a single linear 
%representation shared across all the tasks, this approach is expected to fail. This 
%is for instance the case in which the tasks can be organized in different clusters 
%according to the common linear representation their target vectors share.  

% In order to overcome this issue, previous authors developed non-convex methods attempting at clustering the tasks, see e.g. 
% \cite{argyriou2008algorithm,maurer2012transfer,argyriou2013learning,jacob2009clustered,Andrew}.
To overcome this issue, in this work, we follow the recent literature on heterogeneous meta-learning ~\cite{wang2020structured,vuorio2019multimodal,rusu2018meta,jerfel2019reconciling,cai2020weighted,yao2019hierarchically,denevi2020advantage,bertinetto2018meta} and propose a so-called {\itshape conditional meta-learning} approach for meta-learning a representation. Our algorithm learns a conditioning function 
mapping available tasks' side information into a \emph{linear} representation that is tuned to that task at hand. Our approach borrows from \cite{denevi2020advantage}, where the authors proposed 
a conditional meta-learning approach for fine tuning and biased regularization. In those cases however, the tasks' target vectors are assumed to be all 
close to a common bias vector rather than sharing the same low-dimensional linear 
representation, as instead explored in this work. As we explain in the following, working 
with linear representations brings additional difficulties than working with bias vectors, but, on the other hand, it is also a relevant and effective framework in many scenarios.
%MP: dare una idea (sep- pur vaga) del punto chiave. Quindi rimarcare che 
%rep. learning e’ + importante / centrale che il bias. In un certo senso il bias 
%e’ una rappresentazione 1d

In this work, we propose an online conditional method for linear representation learning with strong theoretical guarantees. In particular, we show that the method is advantageous 
over standard (unconditional) representation learning methods used in meta-learning
when the environment of observed tasks is heterogeneous. 
\paragraph{Contributions and Organization} 
The contributions of this work are the following. First, in 
\cref{Conditional Representation Learning}, we design 
a conditional meta-learning approach to infer a linear 
representation that is tuned to the task at hand. 
Second, in \cref{The advantage of Conditional Representation Learning},
we formally characterize circumstances under which our  
conditional framework brings advantage w.r.t. the standard 
unconditional approach. In particular, we argue that this is the case  
when the tasks are organized in different clusters according 
to the support pattern or linear representation their target vectors' share. 
Third, in \cref{Conditional Representation Meta-Learning Algorithm},
we design a convex meta-algorithm providing a comparable gain as the 
number of the tasks it observes increases. In the unconditional setting, 
the proposed method is able to recover faster rates and it requires to tune 
one less hyper-parameter w.r.t. the state-of-the-art unconditional 
methods. Finally, in \cref{experiments}, we present numerical experiments 
supporting our theoretical claims. We conclude our work in \cref{Conclusion} 
and we postpone the missing proofs to the supplementary material.
%Our work offers four contributions. First, in \cref{unconditional_for_bias_introduction}, 
%we introduce a new conditional 
%meta-learning framework with side information for biased regularization and 
%fine tuning. Second, in \cref{unconditional_for_bias_advantage}, we formally 
%show that, under certain assumptions, this conditional meta-learning approach results 
%to be significantly advantageous w.r.t. the standard unconditional counterpart. 
%We then describe two common settings in which such conditions are 
%satisfied, supporting the potential importance of our study for real-world scenarios.
%Third, in \cref{proposed_method}, we propose a convex meta-algorithm 
%providing a comparable advantage also in practice, as the number of observed 
%tasks increases. Fourth, in \cref{experiments}, we present numerical experiments in 
%which we test our theory and the performance of our method. Our conclusions
%are drawn in \cref{conclusion} and technical proofs are postponed to the appendix.

%---------------------------------------------------------------------------------------------------------------------------------------------

\section{Conditional Representation Learning}
\label{Conditional Representation Learning}

In this section we introduce
our conditional meta-learning setting for representation learning. Then, we proceed to 
identify the differences w.r.t. (with respect to) the standard unconditional counterpart. We begin our overview by first introducing the class of inner learning algorithms considered in this work.

\paragraph{Within-Task Algorithms}
We consider the standard linear supervised learning setting over  
$\Z = \X\times\Y$ with $\X\subseteq\R^d$ and $\Y\subseteq\R$ 
input and output spaces, respectively. We denote by $\P(\Z)$ the set 
of probability distributions (tasks) over $\Z$. For any task $\mu \in {\cal P}(\Z)$ 
and a given loss function $\ell:\R\times\R\to\R$, we aim at finding a weight 
vector $w_\task\in\R^d$ minimizing the {\itshape expected risk}
\begin{equation} \label{single_task_risk}
\min_{w\in\R^d}~\cR_\task(w) 
\quad \quad 
\cR_\task(w) = \Exp_{(x,y)\sim\task}~\ell \bigl(\scal{ x}{w}, y \bigr),
\end{equation}
where, $\scal{\cdot}{\cdot}$ represents the Euclidean product in $\Real^d$. 
In practice, $\task$ is only partially observed trough a dataset 
$\Zn = (x_i, y_i)_{i = 1}^n \sim \task^n$, namely, a collection of $n$ identically 
independently distributed (i.i.d.) points sampled from $\task$. Thus, the goal 
becomes to use a learning algorithm in order to estimate a candidate weight 
vector with a small expected risk converging to the ideal $\cR_\mu(w_\mu)$ 
as the sample size $n$ grows.

Specifically, in this work we will consider as candidate estimators, the family of regularized empirical risk 
minimizers for linear feature learning \cite{argyriou2008convex}. Formally, denoting by $\D = \bigcup_{n\in\N}\Z^n$ the space of all datasets on $\Z$, for a given $\theta\in\Theta$ in $\Theta=\psd$ the set of positive definite $d\times d$ matrices, we will consider the 
%MP: Nitpicking: $\theta$ for a matrix is a bit annoying. I would use $D$ or other
following learning algorithms $A(\theta, \cdot):\D\to\R^d$:
\begin{equation} \label{RERM_feature}
A(\theta, \Zn) = \argmin_{w \in \ran(\theta)\subset\Real^d} ~ \cR_{\Zn,\theta}(w),
\end{equation}
where $\ran(\theta)$ denotes the range of $\theta$ and we defined
\begin{equation}
\cR_{\Zn,\theta}(w)  = \frac{1}{n} \sum_{i = 1}^n \ell(\langle x_i, w \rangle, y_i) 
+ \frac{1}{2} \big \langle w, \theta^\dagger w \big \rangle,
%+ \iota_{\ran(\theta)}(w),
\end{equation}
% $\iota_{\ran(\theta)}$ denotes the indicator function of the range of $\theta$. 
for any $w\in\ran(\theta)$. Here $\theta^\dagger$ denotes the pseudoinverse of $\theta$. Throughout this work we will denote by $\cR_Z(\cdot) = 1/n \sum_{i = 1}^n 
\ell(\langle x_i, \cdot \rangle, y_i)$ the empirical risk associated to $\Zn$. 

Here, $\theta$ plays the role of a linear feature representation that is learned during the meta-learning process \citep[see][for more details on the interpretation]{argyriou2008convex}. 

\begin{remark} [Within-Task Regularization Parameter]
\label{no_lambda}
We observe that, differently to previous work~\citep[see e.g.][]{denevi2019online}, we consider 
the meta-parameters $\theta$ to be any
positive semidefinite matrix, without constraint on its trace (e.g. $\tr(\theta)\leq1$).
This allows us to absorb the regularization parameter
$\la$ typically used to control $\la\scal{w}{\theta^\dagger w}$. This choice is advantageous both in practice since it reduces the number of hyper-parameter to tune and in theory (as discussed in the following) by enjoying faster learning rates.
%As we will see in the following, this will turn out into a method requiring one less hyper-parameter to tune  and relying on faster rates.
\end{remark}

\begin{remark}[Online Variant of \cref{RERM_feature}] \label{online_algorithm_remark}
While in the following we will focus on algorithms of the form of \cref{RERM_feature}, our analysis and results extend also to the setting in which the exact minimization of the empirical risk is replaced by a pre-conditioned variant of online gradient 
descent on $\cR_{\Zn, \theta}$, with starting point $w_0 = 0\in\R^d$ and step size inversely proportional to the iteration:
\begin{equation} \label{online_inner_algorithm}
\begin{split}
&A(\theta, \Zn) = \frac{1}{n} \sum_{i = 1}^n w_i,
\quad \quad 
% & w_{i+1} = w_i - \frac{\theta s_i x_i + w_i}{i} \\
w_{i+1} = w_i - \frac{\theta p_i}{i} \\
& p_i = s_i x_i + \theta^\dagger w_i
\quad \quad
s_i \in \partial \ell(\cdot, y_i)(\langle x_i, w_i \rangle).
\end{split}
\end{equation}
This modification brings additional negligible 
logarithmic factors in our bounds in the following.
%MP and is presented in more details in the Supp Mat?
\end{remark}
\paragraph{Unconditional Meta-Learning} 
The standard unconditional meta-learning setting assumes there
exist a meta-distribution $\rho  \in \P(\Tt)$ -- also 
%(called 
called {\itshape environment} in
\citep{baxter2000model} -- over a family $\Tt\subseteq\P(\Z)$ of distributions 
(tasks) $\task$ and it aims at selecting an inner algorithm in the family above
that is {well suited to solve tasks $\task$ sampled from  $\env$}. 
This target can be reformulated as finding a linear representation 
$\theta_\env \in \Theta$ such that the corresponding
%characterizing an the 
algorithm $A(\theta_\env,\cdot)$ 
%minimizing 
minimizes the \emph{transfer risk}
\begin{equation} \label{meta_learning_problem_1}
\min_{\theta \in \Theta} ~ \ee_\env(\theta)
\qquad 
\ee_\env(\theta) = \Exp_{\task \sim \env} ~ \Exp_{\Zn \sim \task^n}
~ \cR_\task \bigl( A(\theta, \Zn) \bigr).
\end{equation}
% \begin{equation} \label{meta_learning_problem_1}
% \min_{\theta\in\Theta} \left\{ ~ \ee_\env(\theta) = \Exp_{\task \sim \env} ~ \Exp_{\Zn \sim \task^n}
% ~ \cR_\task \bigl( A(\theta, \Zn) \bigr)~\right\}
% \end{equation}
In practice, this stochastic problem is usually tackled by iteratively 
sampling a task $\task\sim\env$ and a corresponding
dataset $\Zn \sim \task^n$, and, then, performing a step of stochastic 
gradient descent on an empirical approximation of \cref{meta_learning_problem_1} 
computed from $\Zn$. This has approach has proven
%Such an approach has been shown to be 
effective for instance when the tasks of the environment share a 
simple common linear representation, see e.g. 
\cite{finn2019online,balcan2019provable,khodak2019adaptive,denevi2019online,pmlr-v70-finn17a,denevi2019learning,finn2018meta,hazan19}.
However, when a single linear representation is not sufficient for the entire 
environment of tasks (e.g. multi-clusters), this homogeneous approach is 
expected to fail. In order to overcome this limitation, some recent works have 
adopted the following conditional approach to the problem, see e.g. 
\cite{wang2020structured,vuorio2019multimodal,rusu2018meta,jerfel2019reconciling,cai2020weighted,yao2019hierarchically,denevi2020advantage}. 

\paragraph{Conditional Meta-Learning}
Analogously to \cite{denevi2020advantage}, we assume that any task $\task \sim \env$
is provided of additional side information $s \in \Ss$. In such a case, we consider the 
environment $\env$ as a distribution $\rho \in \P(\Tt,\Ss)$ over the set $\Tt$ of tasks and the set $\Ss$ of 
possible side information. Moreover, as usual, we assume $\env$ to decompose in $\env(\cdot|s)\env_\Ss(\cdot)$ 
and $\env(\cdot|\mu)\env_\Tt(\cdot)$ the conditional and marginal distributions w.r.t. $\Ss$ and $\Tt$. 
For instance, we observe that the side information $s$ could contain descriptive features 
of the associated task, for example attributes in collaborative filtering \cite{abernethy2009new}, or additional 
information about the users in recommendation systems \cite{harper2015movielens}). Moreover $s$ could be 
formed by a portion of the dataset sampled from $\task$ (see \cite{wang2020structured,denevi2020advantage}).
Conditional meta-learning leverages this additional side information in order to 
adapt (or condition) the linear representation $\theta \in \Theta$ on the associated task 
at hand, by learning a linear-representation-valued function $\tau$ solving the problem 
\begin{equation} \label{conditional_meta_learning_problem_1}
\min_{\tau \in \T}{\ee_\env(\tau)},
\qquad
%\quad 
\ee_\env(\tau) {=} \Exp_{(\task{,} s) \sim \env} \Exp_{\Zn \sim \task^n} 
\cR_\task({{A}(\tau(s){,}\Zn)})%,
\end{equation}
over the space $\T$ of measurable functions $\tau: \Ss \to \Theta$. Notice that 
we retrieve the unconditional meta-learning problem in \cref{meta_learning_problem_1} 
if we restrict \cref{conditional_meta_learning_problem_1} to the set of functions  
$\T^{\rm const} = \{\tau ~|~ \tau(\cdot) \equiv \theta,~ \theta\in\Theta\}$, mapping all 
the side information into the same constant linear representation. 

In the next section, we will investigate the theoretical advantages of adopting such a conditional perspective 
and, then, we will introduce a convex meta-algorithm to tackle \cref{conditional_meta_learning_problem_1}. 

%We conclude this section by drawing a connection between our formulation and previous work on the topic.
%\begin{remark}[Datasets as side information] \label{double_sample_size}
%A relevant setting is the case where the side information $s$ corresponds to
%an additional ({\itshape conditional}) dataset $\Zn^{cond}$ sampled from $\mu$, 
%as proposed in \cite{wang2020structured}. We note however that our sampling 
%scheme in \cref{conditional_meta_learning_problem_1} implies that side 
%information $s$ and training set $\Zn$ are independent conditioned on $\task$. 
%Hence, our framework does not allow having $s = \Zn^{cond} = \Zn$, namely, 
%to use the same dataset for both conditioning and training the inner algorithm $A(\tau(\Zn),\Zn)$, 
%as done in \cite{wang2020structured}.
%This is a minor issue since one can always split $\Zn$ in two parts and use one 
%part for training and the other one for conditioning.
%\end{remark}

%-----------------------------------------------------------------------------------------------------------------------------

\section{The Advantage of Conditional Representation Learning}
\label{The advantage of Conditional Representation Learning}

In order to characterize the behavior of the optimal solution of 
\cref{conditional_meta_learning_problem_1} and to investigate the 
potential advantage of conditional meta-learning, we analyze the 
generalization properties of a given conditioning
%conditional 
function $\tau$. 
Formally, we compare the error $\ee_\env(\tau)$ w.r.t. the optimal 
minimum risk
\begin{equation} \label{oracle}
\ee_\env^*  = \Exp_{\task \sim \rho} ~ \cR_\task (w_\task)
\quad \qquad 
w_\task = \argmin_{w \in \Real^d} ~ \cR_\task (w).
\end{equation}
In order to do this, we first need to introduce the following standard 
assumptions used also in previous literature. Throughout this work we 
will denote by $\cdot \trans$ the standard transposition operation. 

\begin{assumption}\label{ass_1}
Let $\ell$ be a convex and $L$-Lipschitz loss function in the first argument. 
Additionally, there exist $\rx>0$ such that $\| x \| \le \rx$ for any $x\in\X$. 
\end{assumption}

\begin{theorem}[Excess Risk with Generic Conditioning Function $\tau$] \label{bound_fixed_feature}
Let \cref{ass_1} hold. For any $s \sim \ms$, introduce the conditional 
covariance matrices
\begin{align} \label{cov_matrices_cond}
\begin{split}
W(s) = \Exp_{\task \sim \env(\cdot|s)} w_\task w_\task \trans,
\quad \quad \quad 
C(s) = \Exp_{\task \sim \env(\cdot|s)} \Exp_{x \sim \eta_\task} x x \trans,
\end{split}
\end{align}
where, $\eta_\task$ denotes the inputs' marginal distribution of the
task $\task$. Let $\tau \in \T$ such that $\ran(W(s)) \subseteq \ran(\tau(s))$ 
for any $s \sim \ms$ and let $A(\tau(s), \cdot)$ be the associated inner 
algorithm from \cref{RERM_feature}. Then,
\begin{equation} \label{bound_general}
\begin{split}
\ee_\env(\tau) - \ee_\env^*  
\le \frac{\Exp_{s \sim \ms} \tr \big( \tau(s)^\dagger W(s) \big)}{2} 
+ \frac{2 L^2  \Exp_{s \sim \ms} \tr \big(\tau(s) C(s) \big)}{n}.
\end{split}
\end{equation}
\end{theorem}

\begin{proof}
For any $(\task,s) \sim \env$, consider the decomposition
\begin{equation}
\ee_\env(\tau) - \ee_\env^*  = \Exp_{(\task, s) \sim \env} 
\big[ \text{B}_{\task,s} + \text{C}_{\task,s} \big], 
\end{equation}
with
\begin{align*}
\text{B}_{\task,s} & = \Exp_{\Zn \sim \task^n} ~ 
\Big [ \cR_\task(A(\tau(s), \Zn))  - \cR_\Zn(A(\tau(s), \Zn)) \Big ]
\text{C}_{\task,s} & = \Exp_{\Zn \sim \task^n} ~ \Big [ \cR_\Zn(A(\tau(s), \Zn)) 
- \cR_\task (w_\task) \Big].
\end{align*}
$\text{B}_{\task,s}$ is the generalization error of the inner algorithm 
$A(\tau(s),\cdot)$ on the task $\mu$. Hence, applying stability arguments 
(see \cref{generalization_error_RERM} in 
\cref{Generalization Bound of the Within-Task Algorithm}), we can write 
\begin{equation*}
\text{B}_{\task,s} \le \frac{2 L^2  \tr \big(\tau(s) \Exp_{x \sim \eta_\task} x x \trans \big)}{n}.
\end{equation*}
Regarding the term $\text{C}_{\task,s}$, for any conditioning function $\tau$ 
such that $w_\task \in \ran(\tau(s))$, we can write
\begin{equation*}
\begin{split}
\text{C}_{\task,s} 
% & = \Exp_{\Zn \sim \task^n} ~ \Big [ \cR_\Zn(A(\tau(s), \Zn)) - \cR_\task (w_\task) \Big] \\
& = \Exp_{\Zn \sim \task^n} ~ \Big [ \min_{w \in \Real^d: w \in \ran(\tau(s))} 
~ \cR_{\Zn,\tau(s)}(w) - \cR_\task (w_\task) \Big] \\
& \le \Exp_{\Zn \sim \task^n} ~ \Big [
~ \cR_{\Zn,\tau(s)}(w_\task) - \cR_\task (w_\task) \Big] \\
& = \frac{\tr \big( \tau(s)^\dagger w_\task w_\task \trans \big)}{2},
\end{split}
\end{equation*}
where, the second equality exploits the definition of the algorithm in \cref{RERM_feature} 
and the first inequality exploits the definition of minimum. The desired statement follows 
by combining the two bounds above, rewriting $\Exp_{(\task,s) \sim \env} = 
\Exp_{s \sim \ms} \Exp_{\task \sim \env(\cdot|s)}$ and observing that the constraint above 
on $\tau$ can be rewritten as follows
\begin{equation*}
\begin{split}w_\task \in \ran(\tau(s)) \text{ for any } (\task,s) \sim \env & \iff \ran(w_\task w_\task \trans) \subseteq \ran(\tau(s)) \text{ for any } (\task,s) \sim \env \\
& \iff \Exp_{\task \sim \env(\cdot|s)} ~ \big[ \ran(w_\task w_\task \trans) \big] 
\subseteq \ran(\tau(s)) \text{ for any } s \sim \ms \\
& \iff \ran \big( \Exp_{\task \sim \env(\cdot|s)} ~ \big[ w_\task w_\task \trans \big] \big)
\subseteq \ran(\tau(s)) \text{ for any } s \sim \ms,
\end{split}
\end{equation*}
where the second and the third equivalences derive from the fact 
that, for any matrices $A, B \in \psd$ and any scalar $c \neq 0$, 
$\ran(A) \subseteq \ran(A+B) = \ran(A) + \ran(B)$ and $\ran(cA) 
= \ran(A)$, see e.g. \cite{hogben2006handbook,hogben2013handbook}.
\end{proof}
%\begin{restatable}[Range Properties]{lemma}{RangeProperties} 
%\label{range_properties}
%\end{restatable}
%\begin{proof}
%We consider the decomposition
%$\ee_\env(\tau) - \ee_\env^*  = \Exp_{(\task, s) \sim \env} 
%\big[ \text{B}_{\task,s} + \text{C}_{\task,s} \big]$, with
%\begin{equation}
%\text{B}_{\task,s} = 
%\Exp_{\Zn \sim \task^n} ~ 
%\Big [ \cR_\task(A(\tau(s), \Zn))  - \cR_\Zn(A(\tau(s), \Zn)) \Big ]
%\end{equation}
%\begin{equation} \label{approximation_error}
%\text{C}_{\task,s} = 
%\Exp_{\Zn \sim \task^n} ~ \Big [ \cR_\Zn(A(\tau(s), \Zn)) - \cR_\task (w_\task) \Big] \le 
%\Exp_{\Zn \sim \task^n} ~ \Big [ \min_{w \in \Real^d} 
%~ \cR_{\Zn,\tau(s)}(w) - \cR_\task (w_\task) \Big].
%\end{equation}
%$\text{B}_{\task,s}$ is the generalization error of the inner algorithm 
%$A(\tau(s),\cdot)$ on the task $\mu$. 
%Hence, applying \cref{ass_1} and the stability arguments in 
%\cref{generalization_error_RERM} in \cref{generalization_error_RERM_sec}, 
%we can write $\text{B}_{\task,s} \le 2 \rx^2 L^2 (\la n)^{-1}$.
%Regarding the term $\text{C}_{\task,s}$, exploiting the definition of the 
%algorithm in \cref{RERM_bias}, we can write
%$\text{C}_{\task,s}
%\le \frac{\la}{2} ~ \| w_\task - \tau(s) \|^2$.
%The desired statement follows by combining the two bounds 
%above and optimizing w.r.t. $\la$.
%\end{proof}
\cref{bound_fixed_feature} suggests that the conditioning function $\tau_*$ minimizing the right hand side of \cref{bound_general} is a good candidate to solve the meta-learning problem. The following result explores this question by showing that such a minimizer admits a closed form solution. The proof  
is reported in \cref{proof_oracle_function}. In the following, we will denote by 
$\| \cdot \|_F$ and $\| \cdot \|_*$ the Frobenius and trace norm of a matrix, respectively.
\begin{restatable}[Best Conditioning Function in Hindsight]{proposition}{OracleFunction} \label{oracle_function}
The conditioning function minimizer and the minimum of the bound presented in
\cref{bound_fixed_feature} over the set 
\begin{equation*}
\left\{\tau \in \T ~~~\middle|~~~
\ran(W(s)) \subseteq \ran(\tau(s)),~ \ms\textrm{-almost surely} \right\},\end{equation*}
are respectively
\begin{equation*} \label{oracle_conditional}
\tau_\env(s) = 
\frac{\sqrt{n}}{2L}~C(s)^{\dagger/2} (C(s)^{1/2} W(s) C(s)^{1/2})^{1/2}C(s)^{\dagger/2}
\end{equation*}
and
\begin{equation} \label{optimal_conditional_bound}
\ee_\env(\tau_\env) - \ee_\env^*  \le 
\frac{2 L \Exp_{s \sim \ms} \big \| W(s)^{1/2} C(s)^{1/2} \big \|_*}{\sqrt{n}}.
\end{equation}
\end{restatable}

This result allows us to quantify the benefits of adopting the conditional feature learning strategy. 

\paragraph{Conditional Vs. Unconditional Meta-Learning}
Applying \cref{oracle_function} to $\T^{\rm const}$, we obtain the excess risk bound for unconditional meta-learning
\begin{equation} \label{optimal_unconditional_bound}
\ee_\env(\theta_\env) - \ee_\env^*  \le 
\frac{2 L \big \| W_\env^{1/2} C_\env^{1/2} \big \|_*}{\sqrt{n}},
\end{equation}
achieved for $\tau(s) \equiv \theta_\rho$ the meta-parameter
\begin{equation} \label{oracle_unconditional}
\theta_\env = \frac{\sqrt{n}}{2L} C_\env^{\dagger/2} (C_\env^{1/2} 
W_\env C_\env^{1/2})^{1/2}C_\env^{\dagger/2},
\end{equation}
with unconditional covariance matrices
\begin{equation} \label{cov_matrices}
W_\env = \Exp_{\task \sim \env} w_\task w_\task \trans,
\quad \quad
C_\env = \Exp_{\task \sim \env} \Exp_{x \sim \eta_\task} x x \trans.
\end{equation}

We observe that in the previous literature \cite{denevi2018,denevi2019online} the 
authors restricted the unconditional problem over the smaller class of linear 
representation $\hat \Theta = \{ \theta \in \psd: \ran(W_\env) \subseteq \ran(\theta), 
\tr(\theta) \le 1 \}$ and they considered as the best unconditional representation, 
the matrix minimizing only a part of the previous bound, namely, 
\begin{equation} \label{previous_unconditional_oracle}
\hat \theta_\env = \argmin_{\theta \in \hat \Theta}
\tr \big( \theta^\dagger W_\env \big)
= \frac{W_\env^{1/2}}{\tr \big( W_\env^{1/2} \big)}.
\end{equation}
On the other hand, the unconditional oracle we introduce above in 
\cref{oracle_unconditional} allows us to recover a tighter bound which
is able to recover the best performance between independent task learning 
(ITL) and the oracle considered in previous literature \cite{denevi2019online}. 
Indeed, by exploiting the duality between the trace norm $\| \cdot \|_*$
and the operator norm $\| \cdot \|_\infty$ of a matrix, we can upper 
bound the right-side-term in \cref{optimal_unconditional_bound} 
by the quantity 
\begin{equation*}
%\ee_\env(\theta_\env) - \ee_\env^*  \le 
\frac{2 L \min \Big \{ \big \| W_\env^{1/2} \big \|_* \big \| C_\env^{1/2} \big \|_\infty,
\big \| W_\env^{1/2} \big \|_F \big \| C_\env^{1/2} \big \|_F \Big \}}{\sqrt{n}},
\end{equation*}
namely, the minimum between the bound for independent task learning
and the bound for unconditional oracle obtained by previous authors.
Notice that the unconditional quantity in \cref{optimal_unconditional_bound} 
is always bigger than the conditional quantity in \cref{optimal_conditional_bound},
since \cref{optimal_unconditional_bound} coincides with the minimum over
a smaller class of function. In order to quantify the gap between 
these two quantities -- namely, the advantage in using the conditional approach w.r.t. 
the unconditional one -- we have to compare the term $\big \| W_\env^{1/2} C_\env^{1/2} \big \|_*$ 
with the term $\Exp_{s \sim \ms} \big\| C(s)^{1/2} W(s)^{1/2} \big \|_*$.
%\MP{Non ho capito - il bound sopra cool min e' diverso!}
%\begin{equation}
%0 \le {\rm{Gap}} = \frac{2 L}{\sqrt{n}} \Big( \tr \big( W_\env^{1/2} C_\env^{1/2} \big) 
%- \Exp_{s \sim \ms} \tr \big( C(s)^{1/2} W(s)^{1/2} \big) \Big).
%\end{equation}  

We report below a setting that can be considered illustrative 
for many real-world scenarios in which such a gap in performance 
is significant. We refer to \cref{example_proof} for the details and the 
deduction. 
%In the example below, we parametrize each task with 
%the triplet $\task = (w_\task,\eta_\task,\xi_\task)$, where $w_\task$ 
%is the target weight vector, $\eta_\task$ is the marginal distribution 
%on the inputs, $\xi_\task$ is a noise model and $y\sim\mu(\cdot|x)$ 
%is $y = \scal{w_\mu}{x} + \epsilon$ with $x\sim\eta_\task$ and 
%$\epsilon\sim\xi_\task$. 
%Additionally, we denote by $\mathcal{N}(v,\sigma^2 I)$ a Gaussian 
%distribution with mean $v\in\R^d$ and covariance matrix $\sigma^2 I$, 
%with $I$ the $d \times d$ identity matrix. 

\begin{restatable}[Clusters]{example}{ClustersExample} 
\label{clusters_ex}
Let $\Ss = \Real^q$ be the side information space, for some integer $q > 0$.
Let $\env$ be such that the side information marginal distribution $\env_\Ss$ 
is given by a uniform mixture of $m$ uniform distributions. More precisely, let 
$\env_\Ss = \frac{1}{m} \sum_{i=1}^m \env_\Ss^{(i)}$, with $\env_\Ss^{(i)} 
= \mathcal{U} \big( \mathcal{B}(a_i, 1/2) \big)$ the uniform distribution on 
the ball of radius $1/2$ centered at $a_i \in \Ss$, characterizing 
the cluster $i$. For a given side information $s$, a task $\mu \sim \env(\cdot|s)$ 
is sampled such that: $1)$ its inputs' marginal $\eta_\task$ is a distribution 
with constant covariance matrix $C(s) = \Exp_{\mu \sim \env(\cdot|s)} 
\Exp_{x \sim \eta_\mu} x x \trans = C$, for some $C \in \psd$, 
$2)$ $w_\task$ is sampled from a distribution with conditional covariance 
matrix $W(s) = \Exp_{\mu \sim \env(\cdot|s)} w_\task w_\task \trans$, with 
$W(s)$ such that ($C^{1/2} W(s) C^{1/2}) (C^{1/2} W(p) C^{1/2}) = 0$ 
if $s \ne p$.
%Let $\env_\Tt = \frac{1}{m} \sum_{i=1}^m \env_\Tt^{(m)}$ be a uniform mixture of $m$ 
%environments (clusters) of tasks. For each $i=1,\dots,m$, a task $\task \sim\env_\Tt^{(i)}$ 
%is sampled such that: 
%$1)$ $w_\mu\sim\mathcal{N}(w(i),\sigma_w^2I)$ with $w(i) \in\R^d$ a cluster's 
%mean vector and $\sigma_w^2I$ a covariance matrix, 
%$2)$ $\eta_\task = \mathcal{N}(x(i),\sigma_\X^2)$ with mean vector $x(i)\in\R^d$ 
%and variance $\sigma_\X^2$, 
%$3)$ the side information is an $n$ i.i.d. sample from $\eta_\task$, namely 
%$s = (x_i)_{i=1}^n \sim \eta_\task^n$. 
Then, %the gap between conditional and unconditional variance is 
\begin{equation*}
\Exp_{s \sim \ms} \big \| C(s)^{1/2} W(s)^{1/2} \big \|_* 
= \frac{1}{\sqrt{m}} \big \| W_\env^{1/2} C_\env^{1/2} \big \|_*.
\end{equation*}
\end{restatable}

The inequality above tells us that, in the setting of \cref{clusters_ex}, the conditional 
approach gains a $\sqrt{m}$ factor in comparison to the unconditional
approach. Therefore, the larger the number of clusters is, the more pronounced the advantage of conditional approach w.r.t. the 
unconditional one will be. We observe that a particular case of the setting 
above could be that one in which $q = 1$ and the side information
are \emph{noisy} observations of the index of the cluster the tasks 
belong to. In our experiments, in \cref{experiments}, we consider
a more interesting and realistic variant of the setting above, in which
we will use as task's side information a training dataset sampled from 
that task.
In the next section, we introduce a convex meta-algorithm mimicking this 
advantage also in practice.

%------------------------------------------------------------------------------------------------------------------------------

\section{Conditional Representation Meta-Learning Algorithm}
\label{Conditional Representation Meta-Learning Algorithm}

To tackle conditional meta-learning in practice we consider a parametrization where the conditioning 
functions that are modeled w.r.t. a given feature map $\Phi:\Ss\to\R^k$ (with $k \in \N$) 
on the side information space. In other words, we consider $\tau: \Ss \to \psd$,
\begin{equation} \label{parametrization}
\tau(\cdot) = \big( M \Phi(\cdot) \big) \trans M \Phi(\cdot) + C,
\end{equation}
for some tensor $M \in \Real^{p \times d \times k}$ ($p \in \N$) and 
matrix $C \in \psd$.

By construction, the above parametrization guarantees us to learn functions taking values in the set of positive semi-definite matrices. However, directly addressing the meta-learning problem poses two issues: first, dealing with tensorial structures might become computationally challenging in practice and second, such parametrization is quadratic in $M$ and would lead to a non-convex optimization functional in practice. To tackle this issue, the following results shows that we can equivalently rewrite the 
conditioning function in the form of \cref{parametrization} by using a matrix in $\mathbb{S}_+^{dk}$.
%\MP{punto chiave beel-up a bit / rendere + exciting? Da dire anche in intro\&abstract?}
This will allows us to implement our method working with matrices 
in $\mathbb{S}_+^{dk}$, instead of tensors in $\Real^{p \times d \times k}$.
Throughout this work, we will denote by $\otimes$ the Kronecker  
product.

\begin{restatable}[Matricial Re-formulation of $\tau_M(s)$]{proposition}{MatricialRewriting} 
\label{matricial_rewriting}
Let $\tau$ be as in \cref{parametrization}. Then, 
\begin{equation}
\tau(s) = \big( I_d \otimes \Phi(s) \trans \big) H_M \big( I_d \otimes \Phi(s) \big) + C,
\end{equation}
where $I_d$ is the identity in $\Real^{d \times d}$ and $H_M$ is
the matrix in $\Real^{dk \times dk}$ defined by the entries
\begin{equation*}
\big( H_M \big)_{(i-1)k + h, (j-1)k +z} = \big \langle 
M(:,i,h), M(:,j,z) \big \rangle %_{\Real^p},
\end{equation*}
with $i, j = 1, \dots, d$ and $h, z = 1, \dots, k$. 
%Moreover, $H_M \in \mathbb{S}_+^{dk}$ and 
%\begin{equation}
%\tr \big(\tau(s)\big) \le \tr \big( H_M \big) \big \| \Phi(s) \big \|_{\Real^k}^2.
%\end{equation}
\end{restatable}

The arguments above motivate us to consider the following set of 
conditioning functions:
\begin{equation}\label{eq:feature-space-T}
\begin{split}
\T_\Phi = \Big \{~ \tau(\cdot) 
= \big( I_d \otimes \Phi(\cdot) \trans \big) H \big( I_d \otimes \Phi(\cdot) \big) + C ~\Big| 
\textrm{such that}~ H \in \mathbb{S}_+^{dk}, C \in \psd \Big \}.
\end{split}
\end{equation}
To highlight the dependency of a function $\tau \in \T_\Phi$ w.r.t. its parameter $H$
and $C$, we will denote $\tau = \tau_{H,C}$. Evidently, $\T_\Phi$ 
contains the space of all unconditional estimators $\T^{\rm const}$.
We consider $\T_\Phi$ equipped with the canonical norm $\nor{\tau_{H,C}}^2 
= \nor{(H,C)}_F^2 = \nor{H}_F^2 + \nor{C}_F^2$, where, recall, $\| \cdot \|_F$ 
denotes the Frobenius norm. The following two standard assumptions will 
allow us to design and analyse our method.
\begin{assumption} \label{ass_2}
The optimal function $\tau_\rho$ belongs to $\T_\Phi$, namely there 
exist $H_\env \in \mathbb{S}_+^{dk}$ and $C_\env \in \psd$, such that 
$\tau_\env(\cdot) = \tau_{H_\env,C_\env}(\cdot) = \big( I_d \otimes \Phi(\cdot) 
\trans \big) H_\env \big( I_d \otimes \Phi(\cdot) \big) + C_\env$.
\end{assumption}

\begin{assumption} \label{ass_3}
There exists $K>0$ such that $\| \Phi(s) \| \le K$ for any $s\in\Ss$. 
\end{assumption} 

Here, \cref{ass_2} allows us to restrict the conditional meta-learning problem
in \cref{conditional_meta_learning_problem_1} to $\T_\Phi$, rather than 
to the entire space $\T$ of measurable functions, while 
\cref{ass_3} ensures that 
the meta-objective is Lipschitz (see below).

\paragraph{The Convex Surrogate Problem}
We start from observing that, exploiting the generalization properties of the
within-task algorithm (see \cref{generalization_error_RERM} in 
\cref{Generalization Bound of the Within-Task Algorithm}), for any $\tau$, 
we can write the following
\begin{equation*}
\begin{split}
\Exp_{\Zn \sim \task^n} ~ \Big [ \cR_\task(A(\tau(s), \Zn))  \Big ] \le 
& \Exp_{\Zn \sim \task^n} ~ \Big [ \cR_\Zn(A(\tau(s), \Zn)) \Big ] + 
\frac{2 L^2  \tr \big(\tau(s) \Exp_{x \sim \eta_\task} x x \trans \big)}{n} \\
& \le \Exp_{\Zn \sim \task^n} ~ \Big [ \cR_{\Zn,\tau(s)}(A(\tau(s), \Zn)) \Big ] + 
\frac{2 L^2  \tr \big(\tau(s) \Exp_{x \sim \eta_\task} x x \trans \big)}{n},
\end{split}
\end{equation*}
where in the second inequality we have exploited the fact that 
the within-task regularizer is non-negative. Consequently, 
by taking the expectation w.r.t. $(\task,s) \sim \env$ and exploiting 
the fact that the points are i.i.d., 
%and recalling the definitions
%\begin{equation}
%\begin{split}
%\ee_\env(\tau) & ~=~ \Exp_{(\task, s) \sim \env} ~ \Exp_{\Zn \sim \task^n} 
%~ \cR_\task \bigl( A(\tau(s), \Zn) \bigr) \\
%\hat \ee_\env(\tau) & ~=~ \Exp_{(\task, s) \sim \env} ~ \Exp_{\Zn \sim \task^n}
%~~ \cR_{\Zn,\tau(s)}(A(\tau(s),Z)), 
%\end{split}
%\end{equation}
we get 
\begin{equation}
\begin{split}
\ee_\env(\tau) \le 
\Exp_{(\task,s) \sim \env} ~ \Exp_{\Zn \sim \task^n} ~ \bigg[ \cR_{\Zn,\tau(s)}(A(\tau(s), \Zn)) 
+ \frac{2 L^2}{n} \tr \Big(\tau(s) \frac{X \trans X}{n} \Big) \bigg], 
\end{split}
\end{equation}
where $X \in \Real^{n \times d}$ is the matrix with the inputs vectors 
$(x_i)_{i = 1}^n$ as rows.
The inequality above suggests us to introduce the surrogate problem 
\begin{equation} \label{surrogate}
\begin{split}
\min_{\tau \in \T} ~  \hat \ee_\env(\tau) 
\quad \quad \hat \ee_\env(\tau) = \Exp_{(\task,s) \sim \env} ~ 
\Exp_{\Zn \sim \task^n} ~ \bigg[ 
\cR_{\Zn,\tau(s)}(A(\tau(s), \Zn)) 
+ \frac{2 L^2}{n} \tr \Big(\tau(s) \frac{X \trans X}{n} \Big) \bigg],
\end{split}
\end{equation}
where, from the last inequality above, for any $\tau$, we have
\begin{equation} \label{majorization}
\ee_\env(\tau) \le \hat \ee_\env(\tau).
\end{equation}
We stress that the surrogate problem we take here is 
different from the one considered in previous work 
\cite{denevi2019learning,denevi2019online,hazan19}, where the
authors considered as meta-objective only a part of the function 
above, namely, $\Exp_{(\task,s) \sim \env} ~ \Exp_{\Zn \sim \task^n} ~ 
\big[ \cR_{\Zn,\tau(s)}(A(\tau(s), \Zn)) \big]$. 
As we will see in the following, 
%we realized that 
such a choice is more 
appropriate for the problem at hand, since, differently from the meta-objective
used in previous literature,
%\MP{poco exciting - dire qualcosa di piu'?} 
it will allow us to develop a conditional meta-learning
method that is theoretically grounded also for linear representation learning. 

%Following a similar strategy to the one adopted for the unconditional setting 
%in \cite{denevi2019online}, we introduce the following surrogate problem for 
%the conditional one:
%\begin{equation} \label{surrogate}
%\min_{\tau \in \T} ~ \hat \ee_\env(\tau)
%\quad \quad \quad 
%\hat \ee_\env(\tau) ~=~ \Exp_{(\task, s) \sim \env} ~ \Exp_{\Zn \sim \task^n}
%~~ \cR_{\Zn,\tau(s)}(A(\tau(s),Z)),
%\end{equation}
%where we have replaced the inner expected risk $\cR_\task$ with the regularized 
%empirical risk $\cR_{\Zn, \theta}$ in \cref{RERM_bias}. 
%Note that since $A$ is the empirical risk minimizer, we are evaluating $\cR_Z^\la$ \mpp{at its minimizer}. This approach yields a convex function in $\tau$ and has been adopted in the unconditional setting in \cite{denevi2019learning, denevi2019online}.
Exploiting \cref{ass_2}, the surrogate problem in \cref{surrogate} can 
be restricted to the class of linear functions $\T_\Phi$ in \cref{eq:feature-space-T} 
and it can be rewritten more explicitly as %follows
\begin{equation} \label{surrogate_linear}
\begin{split}
&\min_{H \in \mathcal{S}^{d k}, C \in \psd} ~ \Exp_{(\task, s) \sim \env}
~ \Exp_{\Zn \sim \task^n} ~ \LL\big (H, C, s, \Zn \big ) \\
\LL\big (H, C, s, \Zn \big ) = & ~ \cR_{\Zn, \tau_{H,C}(s)}(A(\tau_{H,C}(s),Z))
+ \frac{2 L^2}{n}  \tr \Big(\tau_{H,C}(s) \frac{X \trans X}{n} \Big).
\end{split}
\end{equation}
%\MP{Eqs. \eqref{surrogate} and \eqref{surrogate_linear} are conceptually similar. 
%I would try to merge them}
In the following proposition we outline some useful properties of the 
meta-loss $\LL\big (\cdot, \cdot, s, \Zn \big )$ introduced above 
(such as convexity) supporting its choice as surrogate meta-loss. 

\begin{restatable}[Properties of the Surrogate Meta-Loss $\LL$]{proposition}{PropertiesSurrogate} 
\label{properties_surrogate}
For any $\Zn \in \D$ and $s \in \Ss$, the function 
$\LL\big (\cdot, \cdot, s, \Zn \big )$ 
is convex and one of its subgradients is given,
for any $H \in \mathbb{S}_+^{dk}$ and $C \in \psd$, by
\begin{equation} \label{gradient1}
\begin{split}
& \nabla \LL\big (H, \cdot, s, \Zn \big )(C) = \hat \nabla \\
& \nabla \LL\big (\cdot, C, s, \Zn \big )(H) = \big( I_d \otimes \Phi(s) \big) \hat \nabla 
\big( I_d \otimes \Phi(s) \trans \big) 
\end{split}
\end{equation}
where
\begin{equation*}
\hat \nabla = - \frac{\la}{2} \tau_{H,C}(s)^\dagger w_{\tau_{H,C}(s)}
w_{\tau_{H,C}(s)} \trans \tau_{H,C}(s)^\dagger + \frac{2 L^2 X \trans X}{n^2}.
\end{equation*}
Moreover, under \cref{ass_1} and \cref{ass_3}, we have
\begin{equation*} \label{bound_grad_norm}
\big \| \nabla \LL\big (\cdot, \cdot, s, \Zn \big )(H,C)  \big \|_F 
\le (1+ K^2)(LR)^2 \Big( \frac{1}{2} + \frac{2}{n}\Big).
\end{equation*}
\end{restatable}

The proof of \cref{properties_surrogate} is reported in 
\cref{properties_surrogate_proof}. It follows from combining 
results from \cite{denevi2019online} with the composition of 
the linear parametrization of the functions $\tau_{H,C} \in \T_\Phi$.

\paragraph{The Conditional Meta-Learning Estimator} 
The meta-learning strategy we propose consists in applying Stochastic Gradient Descent (SGD) on the 
surrogate problem in \cref{surrogate_linear}. Such a meta-algorithm is implemented in \cref{meta_alg}: 
we assume to observe a sequence of i.i.d. pairs $(\Zn_t, s_t)_{t = 1}^T$ of training 
datasets and side information, and at each iteration we update the conditional parameters $(H_t,C_t)$ by 
performing a step of constant size $\gamma>0$ in the direction of $-\nabla\LL(\cdot,\cdot,s_t,\Zn_t)(H_t,C_t)$
and a projection step on $\mathbb{S}_+^{dk} \times \psd$. Finally, we output the conditioning function 
$\tau_{\thickbar H, \thickbar C}$ parametrized by $(\thickbar H,\thickbar C)$, the average across all the 
iterates $(H_t,C_t)_{t = 1}^T$. The theorem below analyzes the generalization properties of such a conditioning 
function.

%\begin{algorithm}[tb]
%   \caption{Bubble Sort}
%   \label{alg:example}
%\begin{algorithmic}
%   \State  {\bfseries Input:} data $x_i$, size $m$
%   \REPEAT
%   \State  Initialize $noChange = true$.
%   \FOR{$i=1$ {\bfseries to} $m-1$}
%   \IF{$x_i > x_{i+1}$}
%   \State  Swap $x_i$ and $x_{i+1}$
%   \State  $noChange = false$
%   \ENDIF
%   \ENDFOR
%   \UNTIL{$noChange$ is $true$}
%\end{algorithmic}
%\end{algorithm}

%\makeatletter
%\algrenewcommand\ALG@beginalgorithmic{\footnotesize}
%\makeatother
\begin{algorithm}[tb]
\caption{Meta-Algorithm, SGD on \cref{surrogate_linear}}\label{meta_alg}
\begin{algorithmic}
\State   ~
   \State  {\bfseries Input} ~ $\gamma > 0$ meta-step size, $H_0 \in \mathbb{S}^{dk}_+$,
   $C_0 \in \psd$ 
   \vspace{.05cm}
   \State  {\bfseries Initialization} ~ $H_1 = H_0 \in \mathbb{S}^{dk}_+$, $C = C_0 \in \psd$
  % \State  ~
  \vspace{.04cm}
   \State  {\bfseries For} ~ $t=1$ to $T$
   \vspace{.05cm}
   \State  ~~  Receive $(\task_t, s_t) \sim \env$ and $\Zn_t \sim \task_t^n$
   \vspace{.05cm}
   \State  ~~ Let $\theta_t 
   %= \tau_{H_t,C_t}(s_t) 
   ~=~ \big( I_d \otimes \Phi(s_t) \big) H_t 
   \big( I_d \otimes \Phi(s_t) \trans \big) + C_t$
  % \vspace{.075cm}
   \State  ~~ Compute $w_{\theta_t} ~=~ A(\theta_t, \Zn_t)$ by \cref{RERM_feature}
   \vspace{.05cm}
   \State ~~ Compute $\nabla \LL (\cdot, C_t, s_t, \Zn_t  )(H_t)$ as in \cref{gradient1}
   with $w_{\theta_t}$
   %\vspace{.05cm}
    \State ~~ Compute $\nabla \LL (H_t, \cdot, s_t, \Zn_t  )(C_t)$ as in \cref{gradient1}
   with $w_{\theta_t}$
   %\vspace{.05cm}
   \State  ~~ Update $H_{t+1} = \proj_\Theta \big( H_t - \gamma 
   \nabla \LL (\cdot, C_t, s_t, \Zn_t)(H_t) \big)$
    \vspace{.05cm}
   \State  ~~ Update $C_{t+1} = \proj_\Theta \big( C_t - \gamma 
   \nabla \LL (H_t,\cdot, s_t, \Zn_t)(C_t) \big)$
   \vspace{.075cm}
% \State  ~
 \State  {{\bfseries Return} ~ $\displaystyle {\thickbar H} = \frac{1}{T} \sum_{t=1}^T H_t$,
 $\displaystyle {\thickbar C} = \frac{1}{T} \sum_{t=1}^T C_t$} 
\State  ~
\end{algorithmic}
\end{algorithm}

%\GD{There is only an additional term $1/n$, but it is faster.}
\begin{restatable}[Excess Risk Bound for the Conditioning Function 
Returned by \cref{meta_alg}]{theorem}{BoundEstimatedFeature}\label{bound_estimated_feature}
Let \cref{ass_1} and \cref{ass_3} hold. For any $s \sim \ms$, 
recall the conditional covariance matrices $W(s)$ and $C(s)$ 
introduced in \cref{bound_fixed_feature}. Let $\tau_{H,C}$ be a 
fixed function in $\T_\Phi$ such that $\ran(W(s)) \subseteq \ran(\tau_{H,C}(s))$ 
for any $s \sim \ms$. Let $\thickbar H$ and $\thickbar C$ be the outputs of 
\cref{meta_alg} applied to a sequence $(\Zn_t, s_t)_{t = 1}^T$ of i.i.d. pairs 
sampled from $\env$ with 
%inner regularization parameter and 
meta-step size
\begin{equation}
%\la ~=~ \frac{2 \rx L}{{\rm Var}_\env(\tau_{M,b})}~ \frac{1}{\sqrt{n}}
%\quad \quad \quad 
\gamma ~=~ \frac{\nor{(H-H_0,C-C_0)}_F}{(1+K^2)(LR)^2} ~ 
\Big( \frac{1}{2} + \frac{2}{n} \Big)^{-1} ~ \frac{1}{\sqrt{T}}.
\end{equation}
Then, in expectation w.r.t. the sampling of $(\Zn_t, s_t)_{t = 1}^T$, 
\begin{equation*} \label{feature}
\begin{split}
\Exp ~ \ee_\env(\tau_{\thickbar H,\thickbar C}) - \ee_\env^* 
~\leq~ &\frac{\Exp_{s \sim \ms} \tr \big( \tau_{H,C}(s)^\dagger W(s) \big)}{2} 
+ \frac{2 L^2  \Exp_{s \sim \ms} \tr \big(\tau_{H,C}(s) C(s) \big)}{n} \\
& + \Big( \frac{1}{2} + \frac{2}{n} \Big) \frac{(1 + K^2) (LR)^2 \nor{(H-H_0,C-C_0)}_F}{\sqrt{T}}.
\end{split}
\end{equation*}
\end{restatable}

\begin{proof}[Proof (Sketch)]
The detailed proof is reported in \cref{feature_proof_final_thm}. Exploiting the 
fact that, for any $\tau \in \T$, $\ee_\env(\tau) \le \hat \ee_\env(\tau)$ 
(see \cref{majorization}) and adding $\pm \hat \ee_\env(\tau_{H,C})$, 
we can write the following 
\begin{equation}
\begin{split}
\Exp_{\bf Z} ~ \ee_\env(\tau_{\thickbar H,\thickbar C}) - \ee_\env^*  
\le & \underbrace{\Exp_{\bf Z} ~ \hat \ee_\env(\tau_{\thickbar H,\thickbar C}) 
- \hat \ee_\env(\tau_{H,C})}_{\text{A}(\tau_{H,C})} 
+ \underbrace{\hat \ee_\env(\tau_{H,C}) - \ee_\env^* }_{\text{B}(\tau_{H,C})}.
\end{split}
\end{equation}
The term $\text{A}(\tau_{H,C})$ can be controlled according to the convergence 
properties of the meta-algorithm in \cref{meta_alg} as described in 
\cref{convergence_surrogate}. Regarding the term $\text{B}(\tau_{H,C})$, 
exploiting the definition of the within-task algorithm in \cref{RERM_feature}
as minimum, for any $\tau \in \T$ such that 
$\ran(\Exp_{\task \sim \env(\cdot|s)} w_\task 
w_\task \trans) \subseteq \ran(\tau(s))$ for any $s \sim \ms$, we can rewrite 
\begin{equation*}
\begin{split}
\text{B}(\tau) 
& =  \Exp_{(\task,s) \sim \env} ~ \Exp_{\Zn \sim \task^n} ~ \Big[ \cR_{\Zn,\tau(s)}(A(\tau(s), \Zn))
- \cR_\task(w_\task) \Big] 
+ \frac{2 L^2  \Exp_{(\task,s) \sim \env} ~ \tr \big(\tau(s) \Exp_{x \sim \eta_\task} x x \trans \big)}{n} \\
& \le \frac{\Exp_{(\task,s) \sim \env} ~ \tr \big( \tau(s)^\dagger w_\task w_\task \trans \big)}{2} 
+ \frac{2 L^2  \Exp_{(\task,s) \sim \env} ~ \tr \big(\tau(s) \Exp_{x \sim \eta_\task} x x \trans \big)}{n}.
\end{split}
\end{equation*}
The desired statement then derives from combining the two parts above and 
optimizing w.r.t. $\gamma$.
\end{proof}
We now present some important implications of
%concusions comment about the result we got above in 
\cref{bound_estimated_feature}. 

\paragraph{Proposed Vs. Optimal Conditioning Function}
Specializing the bound in \cref{bound_estimated_feature}
to the best conditioning function $\tau_\env$ in 
\cref{oracle_function}, thanks to \cref{ass_2}, 
we get the following bound for our estimator,
\begin{equation*}
\begin{split}
\Exp ~ \ee_\env(\tau_{\thickbar H,\thickbar C}) - \ee_\env^* & ~\leq~
\mathcal{O} \big( \Exp_{s \sim \ms} \big \| W(s)^{1/2} C(s)^{1/2} \big \|_* ~ n^{-1/2} \big) 
+ \mathcal{O} \big( \|(H_\env - H_0,C_\env-C_0) \|_F ~ T^{-1/2} \big).
\end{split}
\end{equation*}
From such a bound, we can state that our proposed meta-algorithm achieves 
comparable performance to the best conditioning function $\tau_\env$ in hindsight, 
when the number of observed tasks is sufficiently large. Moreover, 
recalling the unconditional oracle $\hat \theta_\env$ in \cref{previous_unconditional_oracle}
used in previous literature, regarding the second term vanishing with $T$, we observe that our 
conditional meta-learning approach incurs a cost of $\|(H_\env-H_0,C_\env-C_0) \|_F T^{-1/2}$ 
as opposed to the cost of $\| \hat \theta_\env-\theta_0 \| T^{-1/4}$ associated to 
state-of-the-art unconditional meta-learning approaches (see 
\cite{denevi2019online,balcan2019provable,khodak2019adaptive,hazan19}).
Thus, our conditional approach presents a faster convergence 
rate w.r.t. $T$ than such unconditional 
methods, but a complexity term that is expected to be larger 
due to the larger complexity of the class of functions we are working
with. 
%We want also to point out that 
Such a faster rate w.r.t. $T$ 
is essentially due to our formulation of the problem on
the entire set of positive-semidefinite matrices (with no trace constraints).
This in fact allows us to 
%englobe 
incorporate the within-task regularization parameter 
$\la$ directly in the linear representation and to gain a $\sqrt{T}$ order 
that was lost in previous literature when tuning w.r.t. the parameter $\la$. At the same time, 
this allows us to develop also a method requiring to tune just one hyper-parameter, 
while previous unconditional approaches requires to tune two hyper-parameters.

\paragraph{Comparison to Unconditional Meta-Learning}
Specializing \cref{bound_estimated_feature} to the best unconditional 
estimator $\tau_{H,C} \equiv \theta_\env$ we introduced in \cref{oracle_unconditional}, 
the bound for our estimator becomes%:
\begin{equation} \label{unconditional_retrieved}
\begin{split}
\Exp ~ \ee_\env(\tau_{\thickbar H,\thickbar C}) - \ee_\env^*  ~\leq~
\mathcal{O} \big( \big \| W_\env^{1/2} C_\env^{1/2} \big \|_* ~ n^{-1/2} \big) 
+ \mathcal{O} \big( \| \theta_\env - C_0 \| ~ T^{-1/2}\big).
\end{split}
\end{equation}
From the bound above, we can conclude 
that the conditional approach provides, at least, the 
same guarantees as its unconditional counterpart.
Moreover, 
%as already discussed, 
we stress again that the bound 
above presents a faster rate w.r.t. $T$ in comparison 
to the state-of-the-art unconditional methods. 
%The motivations are the ones we gave before.

%\begin{remark}
%When $\tau_\env \notin \T_\Phi$ (i.e. when \cref{ass_3} does not hold), 
%our method suffers an additional approximation error.
%%due to the fact $\min_{\tau \in \T_\Phi} ~ {\rm Var}_\env(\tau) > 
%%{\rm Var}_\env(\tau_\env)$.
%%\GD{Why does this solve the problem?}
%In this case, one might nullify the gap above by considering a feature 
%map $\Phi:\Ss\to\hh$ with $\hh$ a universal reproducing kernel Hilbert 
%space of functions. Exploiting standard arguments from online learning 
%with kernels literature (see e.g. \cite{kivinen2004online,singh2012online,shalev2014understanding}),
%in \cref{lemma_only_kernel} in \cref{implementation_kernels} we 
%describe the implementation of \cref{meta_alg} for this setting 
%using only evaluations of the kernel associated to the feature map.
%We leave the corresponding theoretical analysis to future work.
%\end{remark}

%\paragraph{Proposed conditioning function vs ITL}
%Specializing \cref{bound_estimated_feature} to 
%$\tau_{M,b} \equiv 0$ corresponds to force $\gamma = 0$ and, 
%consequently, \cref{meta_alg} to not move. In such 
%a case, we get the bound:
%\begin{equation} \label{comparison_ITL}
%\Exp ~ \ee_\env(\tau_{\thickbar M,\thickbar b}) - \ee_\env^*  ~\leq~ 
%\mathcal{O} \big( {\rm Var}_\env(0) ~ n^{-1/2} \big),
%\end{equation}
%which corresponds to the standard excess risk bound for ITL, 
%see \cite{denevi2019learning, denevi2019online,balcan2019provable,khodak2019adaptive}.
%In other words, our method does not generate negative transfer effect.}
%
\begin{remark}[Online Variant of \cref{RERM_feature}] \label{approx_meta_sub}
Also in this case, as already observed for the bias regularization and 
fine tuning framework proposed in \cite{denevi2020advantage}, when 
we use the online inner family in \cref{online_algorithm_remark}, we 
can approximate the meta-subgradient in \cref{gradient1} by replacing 
the batch regularized empirical risk minimizer $A (\tau_{H, C}(s), \Zn )$ 
in \cref{RERM_feature} with the last iterate of the online algorithm in 
\cref{online_inner_algorithm}. 
%As shown in \cite{denevi2019learning, denevi2019online}
%for the unconditional setting, such an approximation introduces additional 
%logarithmic factors in the bound, but is does not affect the global behavior 
%of the bounds above. 
\end{remark}
%
%\begin{remark}
%\GD{Do we want to specify here (or in the appendix) that this other implementation 
%with kernels does require to memorize the tasks' datasets (hence we are not online)? 
%We can say that there are trick to memorize only part of datasets.}
%When the space in which the image of the feature map lies is high (or even infinite)
%dimensional, we can implement \cref{meta_alg} by using only evaluations 
%of the kernel associated to the feature map $\Phi$, see e.g. \cite{kivinen2004online,singh2012online,shalev2014understanding}.
%We refer to \cref{lemma_only_kernel} in \cref{proofs_Conditional Representation Meta-Learning Algorithm_sec} for
%more details.
%\end{remark}

%-----------------------------------------------------------------------------------------------------------------------------------

\section{Experiments}
\label{experiments}

We now present preliminary experiments in which we compare the 
proposed conditional meta-learning approach in \cref{meta_alg}  (cond.) 
with the unconditional counterpart (uncond.) and solving the tasks independently 
(ITL, namely, running the inner algorithm separately across the tasks 
with the constant linear representation $\theta = I_d \in \psd$).
%\footnote{Code is available at 
%{\small \url{https://github.com/dGiulia/ConditionalMetaLearning.git}}}
We considered regression problems and we evaluated the errors 
by $\ell$ the absolute loss. We implemented the online variant of the 
within-task algorithm introduced in \cref{online_inner_algorithm}. 
The hyper-parameter $\gamma$ was chosen by (meta-)cross validation on separate $T_{\rm tr}$, $T_{\rm va}$ and $T_{\rm te}$ respectively meta-train, -validation and -test sets. Each task
is provided with a training dataset $\Zn_{\rm tr}$ of $n_{\rm tr}$ points 
and a test dataset $\Zn_{\rm te}$ of $n_{\rm te}$ points used to evaluate 
the performance of the within-tasks algorithm. In \cref{experimental_details} 
we report the details of this process in our experiments.

\paragraph{Synthetic Clusters} 
We considered two variants of the setting described in \cref{clusters_ex} with side information corresponding to the training datasets $Z_{\rm tr}$ associated to each task. In both settings, we sampled $T_{\rm tot} = 900$ tasks 
from a uniform mixture of $m$ clusters. For each task $\task$, 
we generated the target vector $w_\task \in \Real^d$ with $d = 20$ as $w_\task = P(j_\task) \tilde w_\task$, where, $j_\task \in \{1, \dots, m\}$ denotes the cluster from which the task $\task$ was sampled and with the components 
of $\tilde w_\task \in\R^{d/(10)}$ sampled from the Gaussian distribution $\mathcal{G}(0,1)$ 
and then $\tilde w_\task$ normalized to have unit norm, with $P(j_\task) \in \R^{d \times d/(10)}$ 
a matrix with orthonormal columns. We then generated the corresponding dataset 
$(x_i,y_i)_{i=1}^{n_{\rm tot}}$ with $n_{\rm tot} = 80$ according to the linear equation 
$y = \langle x, w_\task \rangle + \epsilon$, with $x$ sampled uniformly on the unit sphere 
in $\mathbb{R}^d$ and $\epsilon$ sampled from a Gaussian  distribution, $\epsilon \sim \mathcal{G}(0,0.1)$. In this setting, the operator norm of the inputs' covariance matrix is small 
(equal to $1/d$) and the weight vectors' covariance matrix of each single cluster is low-rank 
(its rank is $d/(10) = 2$). We implemented our conditional method using the feature map 
$\Phi:\D \to \Real^{2d}$ defined by $\Phi(\Zn) = \frac{1}{{n_{\rm tr}}} 
\sum_{i = 1}^{n_{\rm tr}} \phi(z_i)$, with $\phi(z_i) = {\rm vec}\big( x_i (y_i, 1) \trans \big)$, 
where, for any matrix $A = [a_1, a_2] \in \Real^{d \times 2}$ with columns $a_1, a_2 
\in \Real^d$, ${\rm vec}(A) = (a_1, a_2) \trans \in \Real^{2d}$. 
%This is not completely inline with our theory,
%requiring a new independent set of features as side information, but
%we observed that it was an effective strategy. \\

\begin{figure}[t]
%\vskip 0.2in
\begin{center}
\begin{minipage}[t]{\columnwidth}  
\includegraphics[width=.48\columnwidth]{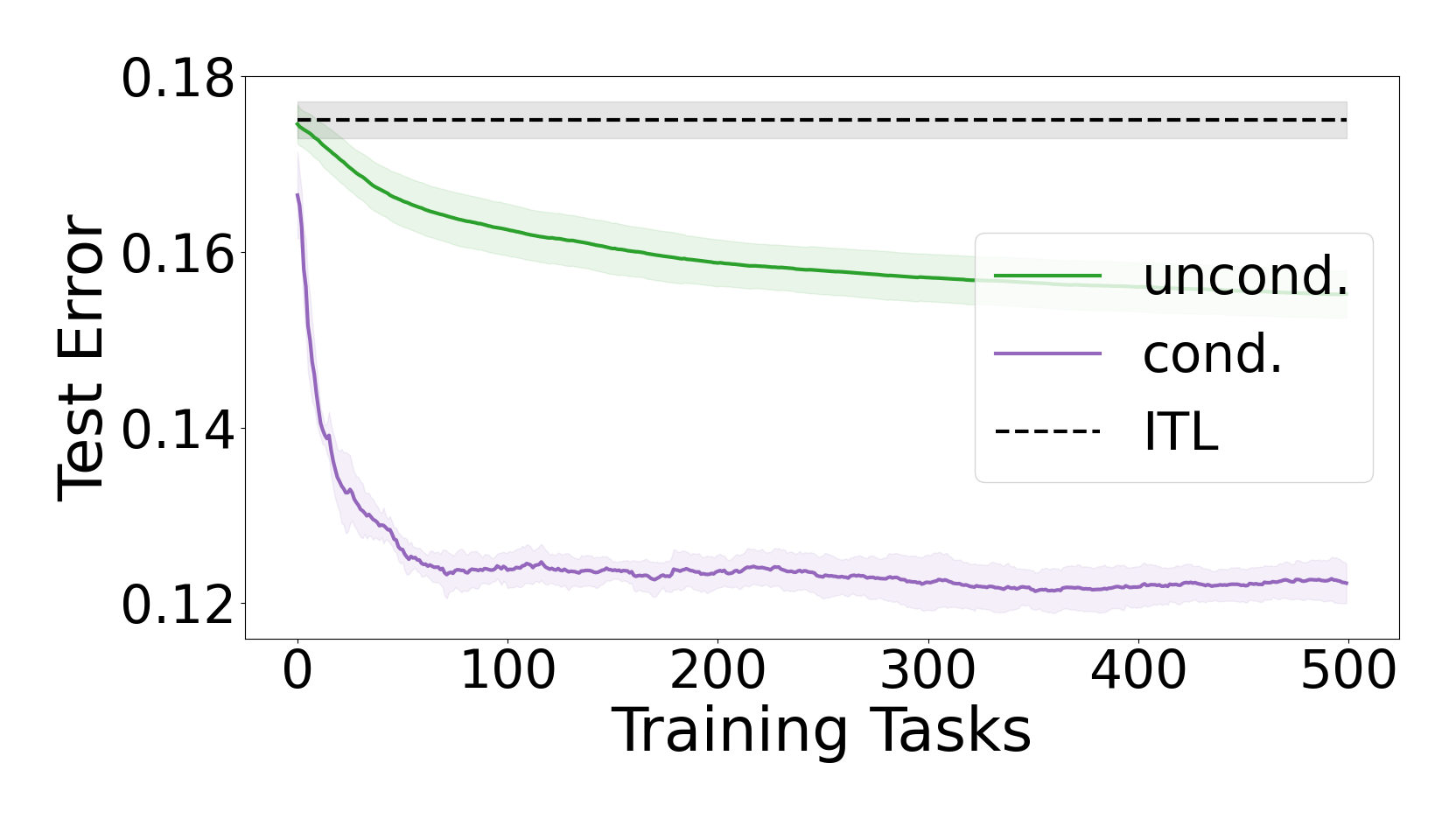} 
%\vskip 0.05in
\quad
\includegraphics[width=.48\columnwidth]{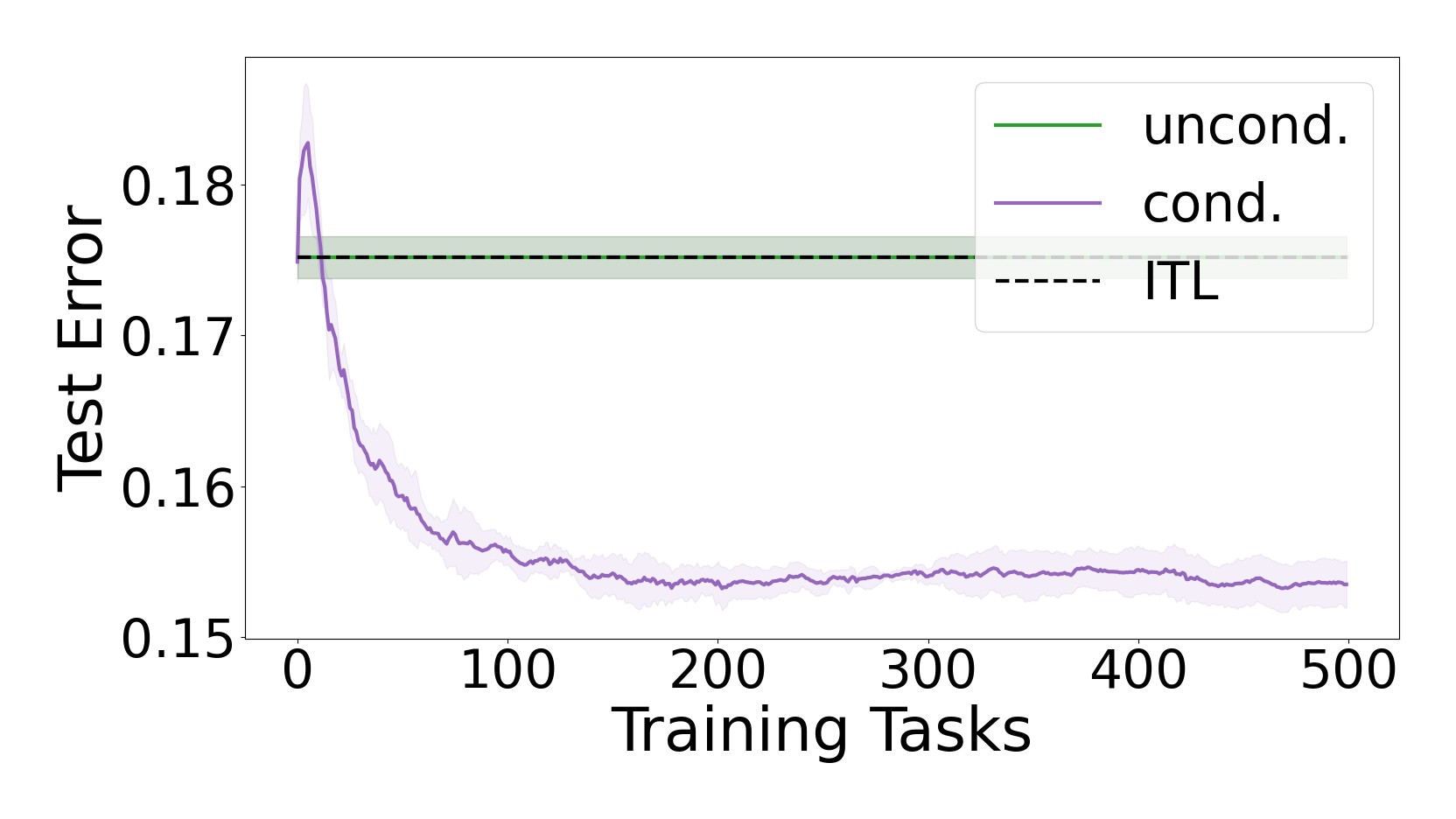}
\end{minipage}
\caption{Test error (averaged over $5$ random generations of the data) 
of different methods w.r.t. an increasing number of tasks
on synthetic data. $2$ clusters (Left) and $6$ clusters (Right).
\label{synth_plots}}
\end{center}
\vskip -0.2in
\end{figure}

In \cref{synth_plots}, we report the results we got 
on an environment of tasks generated as above with $m = 2$ (Left) and $m = 6$ 
(Right) clusters, respectively.  As we can see, when the clusters are two, the unconditional 
approach outperforms ITL (as predicted from previous literature), but the unconditional
method is in turn outperformed by our conditional counterpart. When the number of clusters raises to
six, the performance of unconditional meta-learning degrades to the same performance 
of ITL, while conditional meta-learning outperforms both methods.  
%according to which the advantage of conditional 
%over unconditional meta-learning increases as the number of clusters increases.
Summarizing, the more the heterogeneity of the environment (number of clusters) 
is significant, the more the conditional approach brings advantage w.r.t. the unconditional 
one. This is in line with our statement in \cref{clusters_ex}. 
%When the environment is homogeneous, the performance of the two are equivalent. 

\begin{figure}[t]
%\vskip 0.2in
\begin{center}
\begin{minipage}[t]{\columnwidth}  
\includegraphics[width=.48\columnwidth]{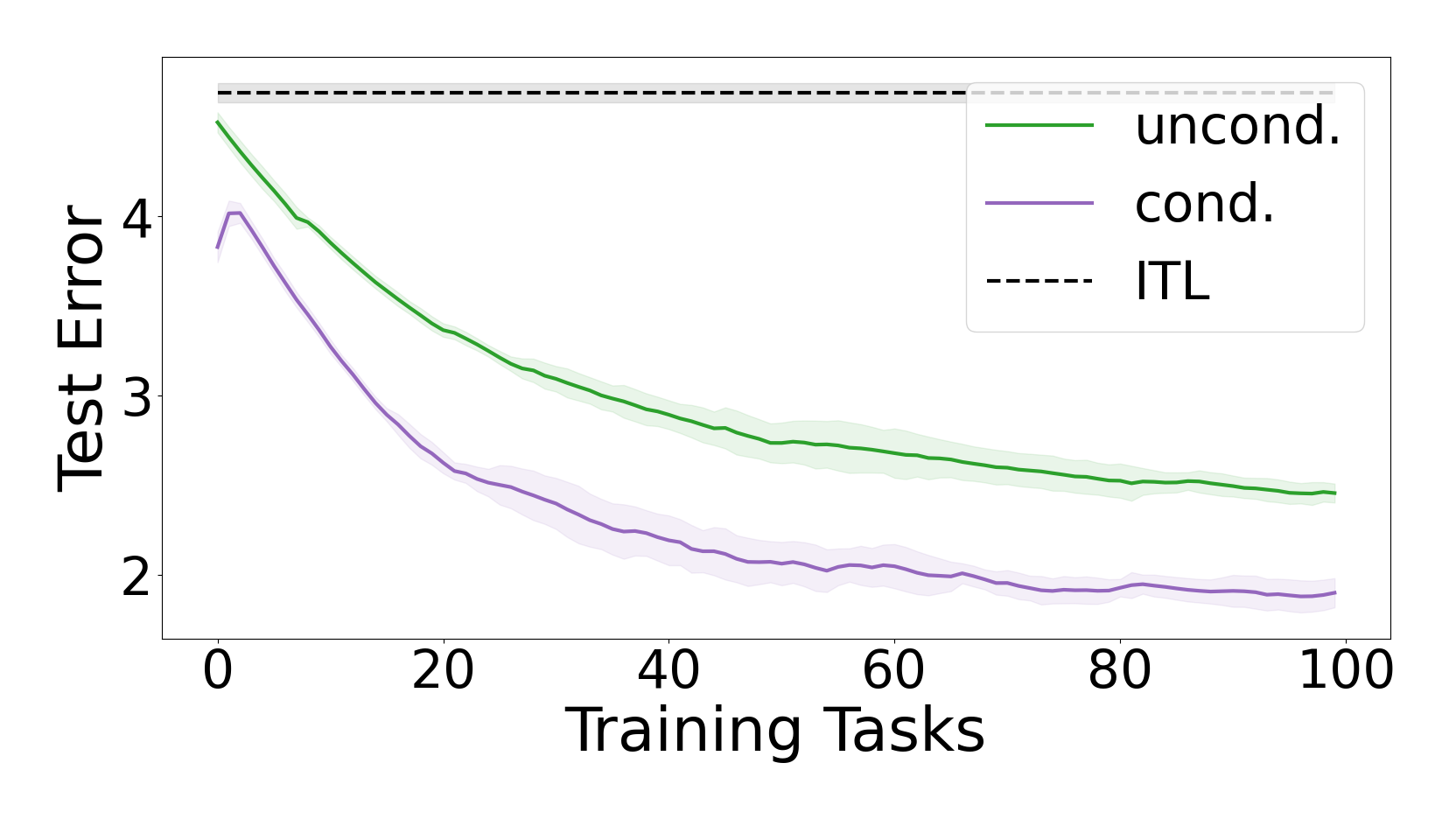} 
%\vskip 0.05in
\quad
\includegraphics[width=.48\columnwidth]{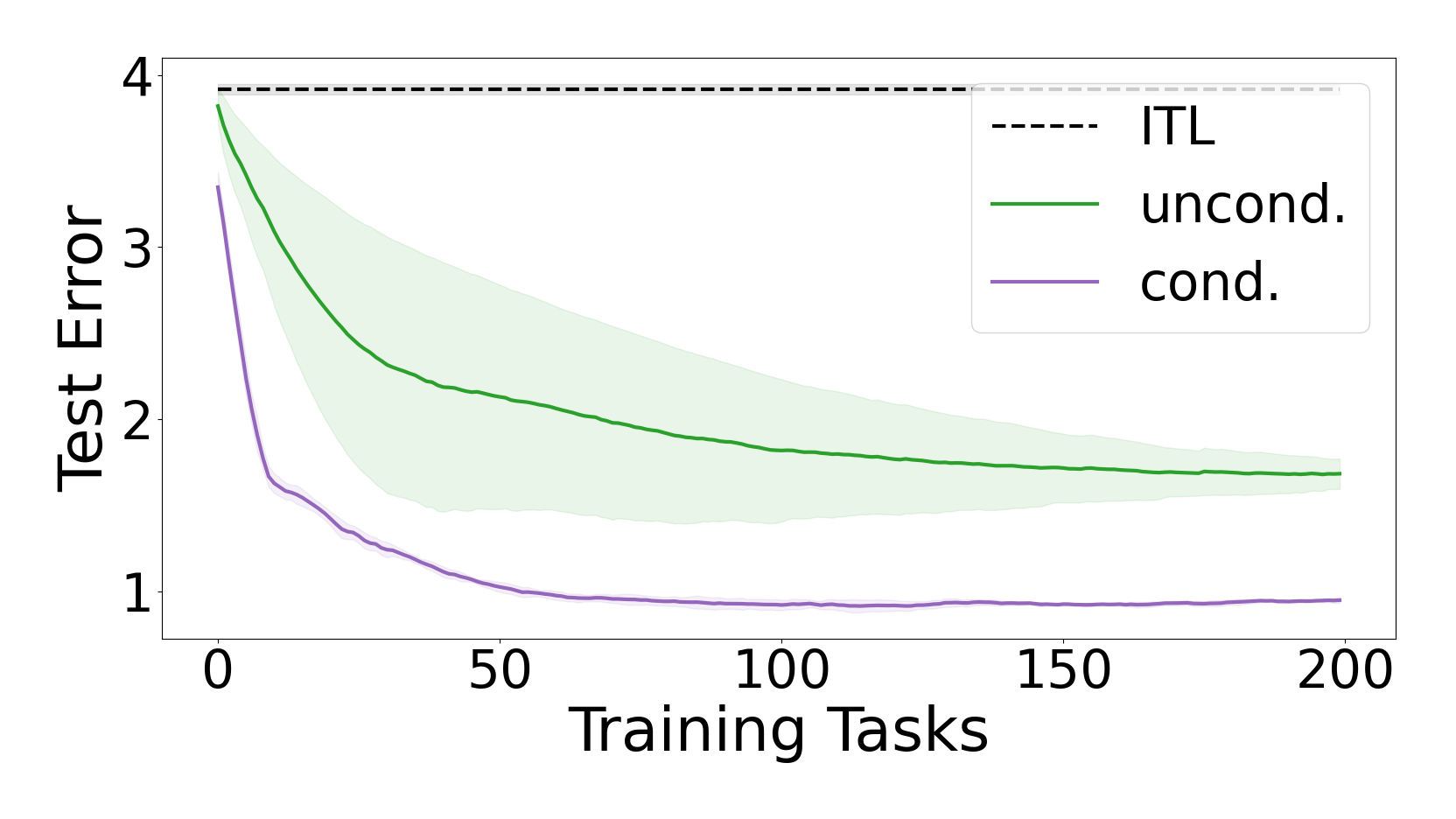}
\end{minipage}
\caption{Test error (averaged over $5$ random splitting of the data) 
of different methods w.r.t. an increasing number of tasks
on the Lenk dataset (Left) and the Movielens-100k dataset (Right).
\label{lenk_movies_plots}}
\end{center}
\vskip -0.2in
\end{figure}

\begin{figure}[t]
%\vskip 0.2in
\begin{center}
\centerline{\includegraphics[width=.48\columnwidth]{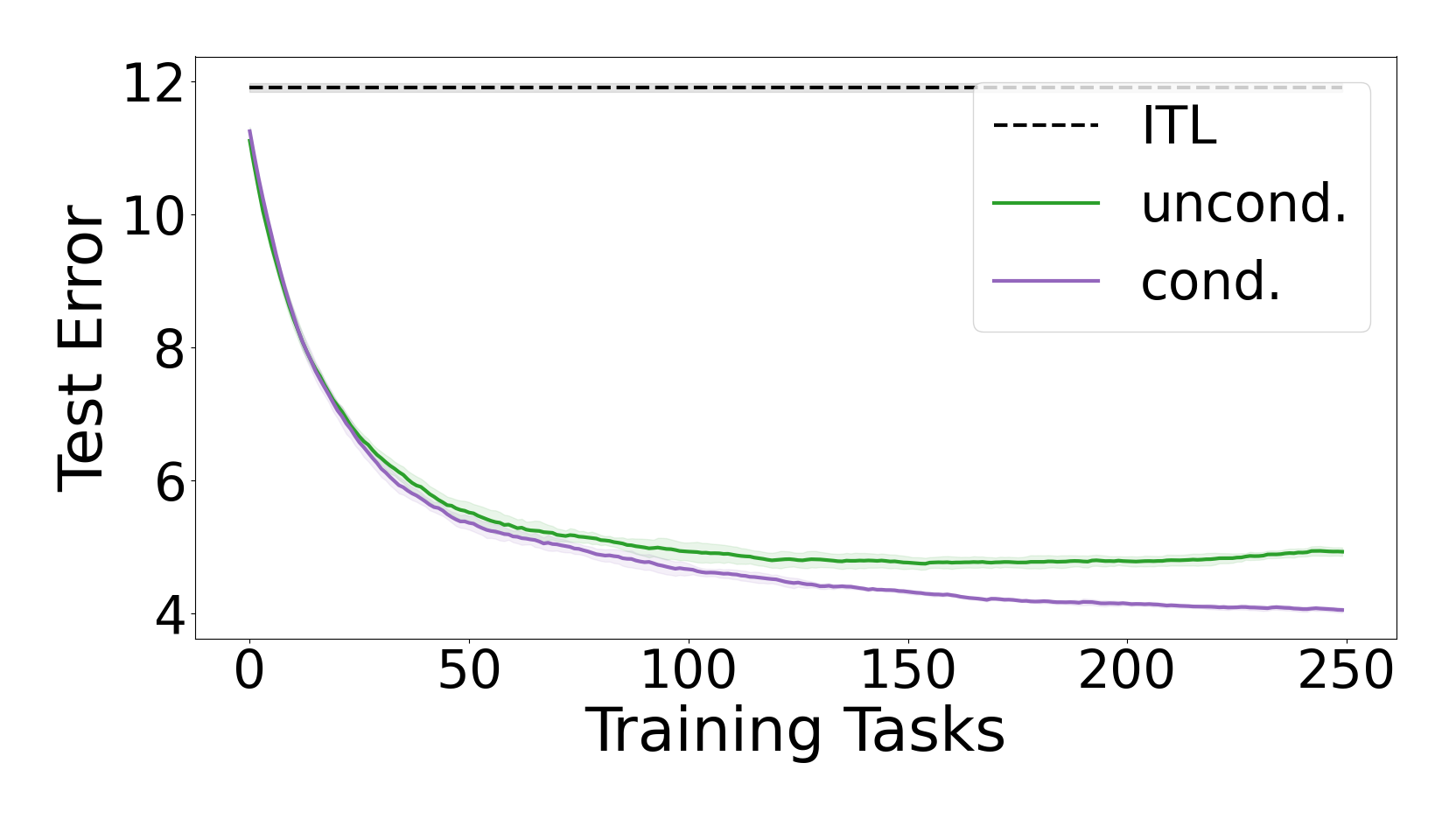}}
\caption{Test error (averaged over $5$ random splitting of the data) 
of different methods w.r.t. an increasing number of tasks
on the Jester-1 dataset.}
\label{jester_plot}
\end{center}
\vskip -0.2in
\end{figure}

\paragraph{Real Datasets} We tested the performance of the methods also on the 
regression problem on the computer survey data from \cite{lenk1996hierarchical} \citep[see also][]{Andrew}.
$T_{\rm tot} = 180$ people (tasks) rated the likelihood of purchasing one 
of $n_{\rm tot} = 20$ computers. The input represents $d = 13$ computers' 
characteristics and the label is a rate in $\{0, \dots,10\}$. In this case, we used as side 
information the training datapoints $\Zn =(z_i)_{i = 1}^{n_{\rm tr}}$ and the feature 
map $\Phi:\D \to \Real^{d+1}$ defined by $\Phi(\Zn) = w_\Zn$, with $w_\Zn$ the 
solution of Tikhonov regularization with the squared loss, namely, the vector satisfying 
$(\hat X \trans \hat X + I_{d+1}) w_\Zn = \hat X \trans y$, where, $\hat X \in \Real^{(d+1) 
\times n}$ is the matrix obtained by adding to the matrix $X \in \Real^{n \times d}$ one 
column of ones at the end. \cref{lenk_movies_plots} (Left) shows that also in this case, the unconditional 
approach outperforms ITL, but the performance of its conditional counterpart is much better. 

Finally, we tested the performance of the methods on the Movielens-100k and Jester-1 real-world datasets, containing ratings of users (tasks) to movies and jokes (points), respectively. We recall that recommendation system settings with $d$ items can be interpreted within the meta-learning setting by considering each data point $(x,y)$ to have input $x\in\R^d$ to be the one-hot encoding of the current item to be rated (e.g. a movie or a joke) and $y\in\R$ the corresponding score \citep[see e.g.][for more details]{denevi2019learning}. We restricted the original dataset to the $n_{\rm tot} = 20$ most voted movies/jokes (as a consequence, by formulation, $d = 20$). We guaranteed each user voted at least $5$ movies/jokes, which led to a total of $T_{\rm tot} = 400/450$ tasks (i.e. users). In both cases, we used as side information the training datapoints $\Zn =(z_i)_{i = 1}^{n_{\rm tr}}$. For the Movielens-100k dataset 
we used the same feature map described for the synthetic clusters experiments in \cref{synth_plots}.
%namely a linear kernel mean embedding on the raw features. 
For the Jester-1 dataset, let $M$ and $m$ denote the maximum and minimum rating value that can be assigned to a joke. We adopted the feature map $\Phi:\D \to \Real^{2d+1}$ such that, for any dataset $\Zn=(x_i,y_i)_{i = 1}^n$, we have
\begin{equation}
    \Phi(\Zn) = \left(\begin{array}{c}\textrm{vec}(\tilde\Phi(\Zn))\\
    1\end{array}\right),
\end{equation}
where $\textrm{vec}$ denotes the vectorization operator (i.e. mapping a matrix in the vector concatenating all its columns) and $\widetilde\Phi:\Zn\to\R^{d\times 2}$ is such that
\begin{equation}
\widetilde\Phi(\Zn) = \left[\cos\left(\sum_{i=1}^n x_i \left(
\frac{\pi}{4}\frac{M - y_i}{M - m}\right)\right),~ \sin\left(\sum_{i=1}^n x_i \left(
\frac{\pi}{4}\frac{M - y_i}{M - m}\right)\right) \right] \odot \left(\sum_{i=1}^n x_i\right),
\end{equation}
with $\odot$ denoting the Hadamard (entry-wise) product broad-casted across both columns.

The rationale behind this feature map is to represent as similar vectors those users with similar scores for the same movies. In particular, each item-score pair observed in training is represented as a unitary vector in $\R_{++}^2$, with the angle depending on the score attributed to that item (the vector corresponds to zero if that movie was not observed at the training time). We noticed that this feature map did not provide significant advantages on the Movielens-100k dataset, while being particularly favorable on the Jester-1 benchmark.

We report the average test errors (and standard deviation) for ITL, conditional and unconditional meta-learning in \cref{lenk_movies_plots} (Right) and \cref{jester_plot} for Movielens-100k and Jester-1, respectively.
%averaged across 5 runs each. 
As it can be noticed, the proposed approach performs significantly better than ITL and its unconditional counterpart. This suggests that groups of users might rely each on similar features (but different from those of other groups) to rate an item in the dataset (respectively a movie or a joke). 

%\paragraph{Feature map by Fourier random features}
%We now describe the feature map mimicking 
%a Gaussian distribution by Fourier random features \cite{rahimi2008random} 
%we used in our synthetic circle experiment and Schools dataset experiment. 
%We recall that, in these cases, we considered as side information the  
%inputs $X = (x_i)_{i = 1}^n$. The feature map above was then defined as 
%$\Phi(X) = \frac{1}{n} \sum_{i = 1}^n \phi(x_i)$, where, $\phi$ was 
%built as follows. 
%We first introduced an integer $k \in \N$ and a constant $\sigma \in \Real$. 
%We then sampled a vector $v \in \Real^k$ from the uniform distribution over 
%$[0,2 \pi]^k$ and a matrix $U \in \Real^{k \times d}$ is 
%sampled from the Gaussian distribution $\mathcal{N}(0,\sigma I)$. 
%We then defined
%\begin{equation}
%\phi(x_i) = \sqrt{\frac{2}{k}} ~ {\rm cos}\big(U x_i + v \big) \in \Real^k,
%\end{equation}
%where ${\cos}(\cdot)$ is applied component-wise to the vector.

%-----------------------------------------------------------------------------------------------------------------------------------

\section{Conclusion}
\label{Conclusion}

We proposed a conditional meta-learning approach aiming 
at learning a function mapping task's side information into a 
linear representation that is well suited for the task at hand.
We theoretically and experimentally showed that the proposed
conditional approach is advantageous w.r.t. the standard 
unconditional counterpart when the observed tasks share 
heterogeneous linear representations. Our investigation allowed
us to develop also a new variant of an unconditional meta-learning
method requiring tuning one less hyper-parameter and relying on faster learning bounds than state-of-the-art unconditional
approaches.

We identify two main directions for future work. A first question left opened by most conditional meta-learning methods is how to design a suitable feature map $\Phi$ when the tasks' training datas is used as side information. Following most previous work \cite{rusu2018meta,wang2020structured} in our experiments we adopted a mean embedding representation. However, given the key importance played by such feature map in \cref{bound_estimated_feature}, it will be worth investigating better alternatives in the future. A second direction is more focused on computations and modeling aspects. In particular it will be valuable to investigate how to predict non-linear 
conditioning functions (similarly to e.g. \cite{bertinetto2018meta,pmlr-v70-finn17a}) and develop more efficient versions of our method, 
using less expensive algorithms to update the positive matrices, 
such as the Frank-Wolfe algorithm used in \cite{hazan19} to deal with unconditional settings.

%---------------------------------------------------------------------------------------------------------------------------------------------

\bibliography{references}
\bibliographystyle{apalike}
%\bibliographystyle{abbrv}
%\bibliographystyle{agsm}
%\bibliographystyle{custom}
%\bibliographystyle{natbib}

%--------------------------------------------------------------------------------------------------------------------------------------------

\newpage 
\onecolumn

\section*{Appendix}

\appendix

%\section{Do \emph{not} have an appendix here}
%
%\textbf{\emph{Do not put content after the references.}}
%%
%Put anything that you might normally include after the references in a separate
%supplementary file.
%
%We recommend that you build supplementary material in a separate document.
%If you must create one PDF and cut it up, please be careful to use a tool that
%doesn't alter the margins, and that doesn't aggressively rewrite the PDF file.
%pdftk usually works fine. 
%
%\textbf{Please do not use Apple's preview to cut off supplementary material.} In
%previous years it has altered margins, and created headaches at the camera-ready
%stage. 

The supplementary material is organized as follows. 
In \cref{Generalization Bound of the Within-Task Algorithm} 
we give the bound on the generalization error of the algorithm 
in \cref{RERM_feature} that we used in various proofs. 
In \cref{proof_oracle_function} we report the proof to get 
the closed form of the best conditioning function $\tau_\env$ 
outlined in \cref{oracle_function}. In \cref{example_proof} we 
report the proof of the statement in \cref{clusters_ex}. In 
\cref{proofs_Conditional Representation Meta-Learning Algorithm_sec}, 
we report the proofs of the statements we used in 
\cref{Conditional Representation Meta-Learning Algorithm} in order 
to prove the expected excess risk bound in \cref{bound_estimated_feature} 
for \cref{meta_alg}. Finally, in \cref{experimental_details} 
we report the experimental details we missed in the main body.

%Finally, in \cref{exp_details}, we report the experimental details 
%we omitted in the main body. 

%--------------------------------------------------------------------------------------------------------------------------------

\section{Generalization Bound of the Within-Task Algorithm}
\label{Generalization Bound of the Within-Task Algorithm}

We now study the generalization error of the within-task algorithm in \cref{RERM_feature}, 
i.e. the discrepancy between the (true) risk and the empirical risk of the corresponding estimator. 
This is done in the following result where we exploit stability arguments, more precisely the so-called hypothesis stability, see \citep[Def. $3$]{bousquet2002stability}.

\begin{restatable}[Generalization Error of the Within-Task Algorithm in \cref{RERM_feature}]{proposition}{generalizationErrorRERM} 
\label{generalization_error_RERM}
Let \cref{ass_1} hold. For a distribution $\task \sim \env$, fix a dataset 
$\Zn = (x_i,y_i)_{i = 1}^n \sim \task^n$.
%and, for any $i \in \{1, \dots, n \}$, fix a datapoint $z_i' = (x_i', y_i') \sim 
%\task$ independent from $\Zn$.
For any $\theta \in \Theta$, let $w_\theta(\Zn)$ be the corresponding 
RERM in \cref{RERM_feature} over $\Zn$. Then, the following generalization 
error bound holds for $w_\theta(\Zn)$:
\begin{equation}
\Exp_{\Zn \sim \task^n} ~ \big[ \cR_\task ( w_\theta(\Zn)) - \cR_{\Zn} 
( w_\theta(\Zn)) \big] \le 
\frac{2 L^2}{n} ~ \tr \big( \Exp_{z \sim \task} ~ \theta x x \trans \big).
\end{equation}
\end{restatable}

\begin{proof}
During this proof, we need to make explicit the dependency of the RERM 
(Regularized Empirical Risk Minimizer) 
$w_\theta$ in \cref{RERM_feature} w.r.t. the dataset $\Zn$. For any 
$i \in \{ 1, \dots, n \}$, consider the dataset $\Zn^{(i)}$, a copy of the original 
dataset $\Zn$ in which we exchange the point $z_i = (x_i, y_i)$ with a new
i.i.d. point $z_i' = (x_i', y_i')$. For a fixed $\theta \in \Theta$, we analyze how 
much this perturbation affects the outputs of the RERM algorithm in \cref{RERM_feature}. 
In other words, we study the discrepancy between $w_\theta(\Zn)$ and 
$w_\theta(\Zn^{(i)})$. We start from observing that, since by 
\cref{ass_1} $\cR_{\Zn,\theta}$ is $1$-strongly convex w.r.t. $\| \cdot \|_\theta
= \sqrt{\big \langle \cdot, \theta^\dagger \cdot \big \rangle}$, 
by growth condition and the definition of the RERM algorithm, we can write the following
\begin{equation}
\begin{split}
& \frac{1}{2} ~ \big \| w_\theta(\Zn^{(i)}) - w_\theta(\Zn) \big \|_\theta^2  \le 
\cR_{\Zn,\theta} ( w_\theta(\Zn^{(i)}) )  - \cR_{\Zn,\theta} ( w_\theta(\Zn) ) \\
& \frac{1}{2} ~ \big \| w_\theta(\Zn^{(i)}) - w_\theta(\Zn) \big \|_\theta^2  \le 
\cR_{\Zn^{(i)},\theta} ( w_\theta(\Zn) )  - \cR_{\Zn^{(i)},\theta} ( w_\theta(\Zn^{(i)})).
\end{split}
\end{equation}
Hence, summing the two inequalities above, we get
\begin{equation} \label{o}
\begin{split}
\big \| w_\theta(\Zn^{(i)}) - w_\theta(\Zn) \big \|_\theta^2 & \le \cR_{\Zn,\theta} ( w_\theta(\Zn^{(i)}) )  
- \cR_{\Zn^{(i)},\theta} ( w_\theta(\Zn^{(i)}) ) + \cR_{\Zn^{(i)},\theta} ( w_\theta(\Zn) ) - 
\cR_{\Zn,\theta} ( w_\theta(\Zn) ) \\
& = \frac{{\text A} + {\text B}}{n},
\end{split}
\end{equation}
where we have introduced the terms
\begin{equation}
\begin{split}
& \text{A} = \ell (  \langle x_i', w_\theta(\Zn)  \rangle, y_i' ) - \ell (  \langle x_i', w_\theta(\Zn^{(i)})  \rangle, y_i' ) \\
& \text{B} = \ell (  \langle x_i, w_\theta(\Zn^{(i)})  \rangle, y_i ) - \ell (  \langle x_i, w_\theta(\Zn)  \rangle, y_i ).
\end{split}
\end{equation}
Now, introducing the subgradients $s_{\theta,i}' \in \partial \ell ( \cdot , y_i' ) (  \langle x_i', w_\theta(\Zn)  \rangle )$ and $s_{\theta,i} \in \partial \ell ( \cdot , y_i ) (  \langle x_i, w_\theta(\Zn^{(i)})  \rangle)$ and
applying Holder's inequality, we can write
\begin{equation} \label{oo}
\begin{split}
& \text A \le \big \langle x_i' s_{\theta,i}', w_\theta(\Zn) - w_\theta(\Zn^{(i)}) \big \rangle
\le \big \| x_i' s_{\theta,i}' \big \|_{\theta,*} ~ \big \| w_\theta(\Zn^{(i)}) - w_\theta(\Zn) \big \|_\theta \\
& \text B \le \big \langle x_i s_{\theta,i}, w_\theta(\Zn^{(i)}) - w_\theta(\Zn) \big \rangle
\le \big \| x_i s_{\theta,i} \big \|_{\theta,*} ~ \big \| w_\theta(\Zn^{(i)}) - w_\theta(\Zn) \big \|_\theta,
\end{split}
\end{equation}
where $\| \cdot \|_{\theta,*} = \sqrt{\big \langle \cdot, \theta \cdot \big \rangle}$ is the 
dual norm of $\| \cdot \|_\theta$. Combining these last two inequalities with \cref{o} 
and simplifying, we get the following
\begin{equation} \label{oooo}
\big \| w_\theta(\Zn^{(i)}) - w_\theta(\Zn) \big \|_\theta \le \frac{1}{n} 
\Bigl( \big \| x_i' s_{\theta,i}' \big \|_{\theta,*} + \big \| x_i s_{\theta,i} \big \|_{\theta,*} \Bigr).
\end{equation}
Hence, combining the first row in \cref{oo} with \cref{oooo}, we can write
\begin{equation}
\begin{split}
\ell (  \langle x_i', w_\theta(\Zn)  \rangle, y_i' ) - \ell ( \langle x_i', 
w_\theta(\Zn^{(i)}) \rangle, y_i' )
\le \frac{1}{n} \Bigl( \big \| x_i' s_{\theta,i}' \big \|_{\theta,*}^2 + \big \| x_i' s_{\theta,i}' \big \|_{\theta,*} ~
\big \| x_i s_{\theta,i} \big \|_{\theta,*} \Bigr).
\end{split}
\end{equation}
Now, taking the expectation w.r.t. $\Zn \sim \task^n$ and $z_i' \sim \task$ of the left side member 
above, according to \citep[Lemma $7$]{bousquet2002stability}, we get 
\begin{equation*}
\Exp_{\Zn \sim \task^n} ~ \Exp_{z_i' \sim \task} ~ \Big[ \ell (  \langle x_i', w_\theta(\Zn)  \rangle, y_i' ) 
- \ell (  \langle x_i', w_\theta(\Zn^{(i)})  \rangle, y_i' ) \Big] =
\Exp_{\Zn \sim \task^n} ~ \Big[ \cR_\task ( w_\theta(\Zn) ) - \cR_{\Zn} ( w_\theta(\Zn) ) \Big].
\end{equation*}
Finally, taking the expectation of the right side member, exploiting the fact that the points 
are i.i.d. according $\task$, we get
\begin{equation}
\Exp_{\Zn \sim \task^n} ~ \Exp_{z_i' \sim \task} ~
\frac{1}{n} \Bigg ( \big \| x_i' s_{\theta,i}' \big \|_{\theta,*}^2 + \big \| x_i' s_{\theta,i}' \big \|_{\theta,*} 
\big \| x_i s_{\theta,i} \big \|_{\theta,*} \Bigg ) \le \frac{2}{n} ~ \Exp_{\Zn \sim \task^n} ~ \Exp_{z_i' \sim \task} 
~ \big \| x_i' s_{\theta,i}' \big \|_{\theta,*}^2,
\end{equation}
where we recall that $s_{\theta,i}' \in \partial \ell ( \cdot , y_i' ) (  \langle x_i', w_\theta(\Zn)  \rangle )$.
Combining the two last statements above, we get
\begin{equation}
\Exp_{\Zn \sim \task^n} ~ \big[ \cR_\task ( w_\theta(\Zn)) - \cR_{\Zn} 
( w_\theta(\Zn)) \big] \le \frac{2}{n} ~ \Exp_{\Zn \sim \task^n} ~ 
\Exp_{z_i' \sim \task} ~ \big \| x_i' s_{\theta,i}' \big \|_{\theta,*}^2.
\end{equation}
Finally, substituting the close form of $\| \cdot \|_{\theta,*}$ and 
observing that, by \cref{ass_1} we have $\big \| x_i' s_{\theta,i}' 
\big \|_{\theta,*}^2 \le L^2 \big \| x_i' \big \|_{\theta,*}^2$,
we get the desired statement:
\begin{equation}
\Exp_{\Zn \sim \task^n} ~ \big[ \cR_\task ( w_\theta(\Zn)) - \cR_{\Zn} 
( w_\theta(\Zn)) \big] \le \frac{2 L^2}{n} ~ \Exp_{z_i' \sim \task} ~ 
\big \langle x_i', \theta x_i' \big \rangle 
= \frac{2 L^2}{n} ~ \tr \big( \Exp_{z \sim \task} ~ \theta x x \trans \big).
\end{equation}
\end{proof}

%------------------------------------------------------------------------------------------------------------------------------------

\section{Proof of \cref{oracle_function}}
\label{proof_oracle_function}

In this section we report the proof to get the closed form 
of the best conditioning function $\tau_\env$ outlined in 
\cref{oracle_function}. In order to do this, we need 
the following results.

\begin{lemma} \label{ortoghonal_w_task}
For any $\task \sim \mt$, define the inputs' covariance matrix $C_\task = 
\Exp_{x \sim \eta_\task} x x \trans$. Then, for any $w_\task \in \argmin_{w \in \Real^d} 
\cR_\task(w)$, the projection $w_{0,\task} = C_\task^\dagger C_\task w_\task$ of 
$w_\task$ onto the range of $C_\task$ is still a minimizer of $\cR_\task$.
\end{lemma}

\begin{proof}
Consider the decomposition of $w_\task$ w.r.t. the range of $C_\task$:
\begin{equation}
w_\task = w_{0,\task} + w^{\perp}
\end{equation}
with $w_{0,\task} = C_\task^\dagger C_\task w_\task$ and $w^\perp \in \R^d$ 
such that $C_\task w^\perp = 0$. We note that, almost surely w.r.t. the points $x \in \R^d$ 
sampled from $\task$, we have $\scal{w^\perp}{x} = 0$. This follows by noting that 
by the orthogonality between $C_\task$ and $w^\perp$, we have
\begin{equation}
0 = \scal{w^\perp}{C_\mu w^\perp} = \Exp_{x\sim\eta_\task} \scal{w^\perp}{xx\trans w^\perp} = \Exp_{x\sim\eta_\task} \scal{x}{w^\perp}^2,
\end{equation}
that can hold only if $\scal{x}{w^\perp}^2=0$ almost surely (a.s.) w.r.t. $\eta_\task$. 
We conclude that $\scal{w_\task}{x} = \scal{w_{0,\task}}{x} + \scal{w^\perp}{x} = 
\scal{w_{0,\task}}{x}$ a.s. w.r.t. $\task$ and, consequently, $\cR_\task(w_\task) = 
\cR_\task(w_{0,\task})$.
\end{proof}

\begin{corollary}\label{cor:conditional-covariance-ranges}
For any $s\in\Ss$, recall the conditional covariance matrices in \cref{bound_fixed_feature}.
Then, $\ran(W(s))\subset\ran(C(s))$, namely the range of the task-vector conditional 
covariance $W(s)$ is always contained in the range of the input conditional covariance $C(s)$. 
\end{corollary}

\begin{proof}
The corollary is a direct consequence of the previous \cref{ortoghonal_w_task}.
The result above guarantees that for any $\task \sim \mt$, the rank-one operator 
$W_\task = w_\task w_\task \trans$ has range contained in the range of $C_\task$. 
Taking the conditional expectations $W(s) = \Exp_{\task \sim \env(\cdot|s)} W_\task$ 
and $C(s) = \Exp_{\task \sim \env(\cdot|s)} C_\task$ maintains this relation unaltered,
giving the desired statement.
\end{proof}

\begin{lemma}\label{lemma:orthogonal-projections-and-pseudoinverse}
Let $P\in\psd$ be an orthogonal projector, namely such that $P = P^2$. Then, for any positive definite matrix $\theta\in\mathbb{S}_{++}^d$, we have $P\theta^{-1}P \succeq (P\theta P)^\dagger$.
\end{lemma}

\begin{proof}
The proof is essentially a corollary of Schur's complement. Let consider the decomposition 
\begin{equation}
    \theta = \underbrace{P\theta P}_{A} + \underbrace{P \theta (I-P)}_{B} + \underbrace{(I-P)\theta P}_{B \trans} + \underbrace{(I-P)\theta(I-P)}_{C}
\end{equation}
where $A,C\in\psd$, $B\in\R^{d\times d}$ and $CB = B\trans C = 0$ since $(I-P)P = P(I-P) = P - P^2 = P-P=0$. Additionally, since $C^\dagger = CC^\dagger C^\dagger = C^\dagger C^\dagger C$, we have that also $AC^\dagger = ACC^\dagger C^\dagger = 0$ and analogously $C^\dagger B = B\trans C^\dagger = 0$. Note that since $\theta$ is invertible, both $A$ and $C$ are full rank.
We now observe a few relevant interactions between the objects above. In particular, we observe that $CC^\dagger B\trans = B\trans$. To see this, first note that 
\begin{equation}
    CC^\dagger B\trans = (I-P)\theta(I-P)\big((I-P)\theta(I-P)\big)^\dagger (I-P)\theta P.
\end{equation}
By taking $D = (I-P)\theta^{1/2}$ and using the properties of the pseudoinverse (e.g. $D = D\trans (DD\trans )^\dagger$), we have
\begin{align}
    CC^\dagger B\trans  & = DD\trans (DD\trans )^\dagger D\theta^{1/2}P\\
    & = DD^\dagger D\theta^{1/2}P\\
    & = D\theta^{1/2}P\\
    & = B\trans .
\end{align}
We now derive an alternative characterization of $\theta$ in terms of $A,B,C$. By adding and removing a term $BC^\dagger B$ to $\theta$, we have
\begin{align}
    \theta & = A + B + B\trans  + C\\
            & = A - BC^\dagger B\trans  + BC^\dagger B\trans  + B + B\trans  + C\\
            & = A - BC^\dagger B\trans  + B + C + (B + C)(C^\dagger B\trans )\\
            & = A - BC^\dagger B\trans  + B + C + (A - BC^\dagger B\trans  + B + C)(C^\dagger B)\\
            & = (A - BC^\dagger B\trans  + B + C)(I + C^\dagger B\trans ),
\end{align}
where we have first used the equality $CC^\dagger B\trans  = B\trans $ and then the ortogonality $AC^\dagger = B\trans C^\dagger = 0$. Following a similar reasoning 
\begin{align}
A - BC^\dagger B\trans  + B + C & = A - BC^\dagger B\trans  + C + BC^\dagger C\\
        & = A - BC^\dagger B\trans  + C + BC^\dagger (A - BC^\dagger B\trans  + C)\\
        & = (I + BC^\dagger)(A - BC^\dagger B\trans  + C)
\end{align}
since $BC^\dagger C = C$ (following the same reasoning used for $B\trans  = CC^\dagger B\trans $) and $AC^\dagger = C^\dagger B = 0$. We conclude that 
\begin{equation}\label{eq:second-characterization-theta}
    \theta = (I + BC^\dagger)(A - BC^\dagger B\trans  + C)(I + C^\dagger B\trans ). 
\end{equation}
We now show that all terms in the equation above are invertible. First note that $(I+BC^\dagger)^{-1} = (I-BC^\dagger)$ and $(I+C^\dagger B\trans )^{-1}=(I+C^\dagger B\trans )$. Moreover, since $\theta\succ0$ and $C(A-BC^\dagger B\trans )=0$, then also $A-BC^\dagger B\trans \succ0$. We have
\begin{equation}
    \theta^{-1} = (I - C^\dagger B\trans )(A - BC^\dagger B\trans  + C)^{-1}(I - B C^\dagger),
\end{equation}
from which we conclude
\begin{align}
    P\theta^{-1}P & = P(A-BC^\dagger B\trans  + C)^{-1}P\\
    & = P\Big((A-BC^\dagger B\trans )^{\dagger} + C^{\dagger}\Big)P\\
    & = P(A-BC^\dagger B\trans )^{\dagger}P\\
    & = (A-BC^\dagger B\trans )^{\dagger}.
\end{align}
Since $BC^\dagger B\trans \succeq 0$, we have $A-BC^\dagger B\trans  \preceq A$ and therefore $(A-BC^\dagger B\trans )^\dagger \succeq A^\dagger$ from which we have
\begin{equation}
    P\theta^{-1}P = (A-BC^\dagger B\trans )^\dagger \succeq A^\dagger = (P\theta P)^\dagger,
\end{equation}
as desired. 
\end{proof}

\begin{proposition} \label{prop:generalization-variational-trace-norm}
Consider two matrices $A,B\in \psd$ 
such that $\ran(A) \subseteq \ran(B)$ and consider the 
following associated problem:
\begin{equation}\label{eq:original-surrogate-problem}
\min_{\theta \in\psd,~\ran(A) \subseteq \ran(\theta)}
~ \tr(\theta^{-1} A) + \tr(\theta B).
\end{equation} 
Then, a minimizer and the corresponding minimum 
of the problem above are given by 
\begin{equation}
\theta_* = B^{-1/2}(B^{1/2} A B^{1/2})^{1/2} B^{-1/2}
\quad \quad 
2 \big \| B^{1/2}A^{1/2} \big \|_*.
\end{equation}
Moreover $\theta_*$ is the unique minimizer such that $\ran(\theta_*)\subset\ran(B)$.
\end{proposition}

\begin{proof}
Let $\Theta = \{\theta\in\psd ~|~ \ran(A)\subset\ran(\theta)\}$ and denote by $F:\Theta\to\R$ the objective functional of the problem in \cref{eq:original-surrogate-problem}, such that for any $\theta\in\Theta$
\begin{equation}
    F(\theta) = \tr(\theta^{-1}A) + \tr(\theta B).
\end{equation}
Note that the sign of inverse is well defined since $\ran(A)\subset\ran(\theta)$. 
We begin the proof by showing that the \cref{eq:original-surrogate-problem} is equivalent to 
\begin{equation}\label{eq:equivalent-surrogate-problem-range-B}
\min_{\theta\in\psd,~\ran(A)\subset\ran(\theta) \subset \ran(B)}
~ \tr(\theta^{-1} A) + \tr(\theta B).
\end{equation} 
To see this, let $P= BB^\dagger$ the orthogonal projector onto the range of $B$. By hypothesis, $A = PAP$ and $B = PBP$. Therefore, for any $\theta\in\mathbb{S}_{++}^d$
\begin{align*}
    F(\theta) & = \tr(\theta^{-1}A) + \tr(\theta B) \\
     & = \tr(P\theta^{-1}PA) + \tr(P\theta PB)\\
     & \geq \tr((P\theta P)^\dagger A) + \tr(P\theta PB)\\
     & = F(P\theta P),
\end{align*}
where we have applied the fact that $P\theta^{-1}P\succeq (P\theta P)^\dagger$ from \cref{lemma:orthogonal-projections-and-pseudoinverse} and the positive semidefinteness of $A$. 
The inequality above implies the equivalence between \cref{eq:original-surrogate-problem} and \cref{eq:equivalent-surrogate-problem-range-B}. Indeed, let $\theta_*\in\Theta$ be a minimizer of \cref{eq:original-surrogate-problem} and consider a sequence $(\theta_n)_{n\in\N}$ such that $\theta_n\in\mathbb{S}_{++}^d$ for any $n\in\N$ and $\theta_n\to\theta_*$. By continuity of $F$ we have also that $F(\theta_n)\to F(\theta_*)$. Clearly, $F(\theta_*)\leq F(P\theta_n P)\leq F(\theta_n)$ and therefore also $F(P\theta_n P)\to F(\theta_*)$. By continuity of $F$ over $\Theta$, this also implies that the limit $\lim_{n\to+\infty}P\theta_n P = P\theta_*P$ is a minimizer for \cref{eq:original-surrogate-problem} (and one such that $\ran(\theta_*)\subset\ran(B)$). 
We consider now the set $\Theta_B=\{\theta\in\psd~|~\ran(\theta) =\ran(B)\}$ of all positive semidefinite matrices with same range as $B$, hence invertible on $\ran(B)$. Note that $\Theta_B$ is an open subset of $\Theta$ and its closure in $\Theta$ corresponds to $\Theta$ itself. By definition, any $\theta\in\Theta_B$ is such that $\theta = B^{\dagger/2}X B^{\dagger/2}$ with $\ran(X)=\ran(B)$. This implies in particular that $XB^{\dagger}B = X$ and $\theta^\dagger = B^{\dagger/2}X^\dagger B^{\dagger/2}$. Therefore,
\begin{align}
    F(\theta) & = \tr(\theta^\dagger A) + \tr(\theta B) \\
    & = \tr(X^\dagger B^{1/2}AB^{1/2}) + \tr(X),
\end{align}
and $\ran(B^{1/2}AB^{1/2})\subseteq \ran(B) = \ran(X)$. We can now minimize the problem w.r.t. $X$, namely
\begin{equation}\label{eq:equivalent-surrogate-problem-change-of-variables}
\min_{X\in\psd,~\ran(B^{1/2}AB^{1/2})\subseteq\ran(X)}
~ \tr(X^\dagger B^{1/2}AB^{1/2}) + \tr(X).
\end{equation} 
The minimization corresponds to the variational form of the trace norm of $B^{1/2}AB^{1/2}$ \cite{micchelli2013regularizers} and has solution $X_* = (B^{1/2}AB^{1/2})^{1/2}$, with minimum corresponding to $2\tr((B^{1/2}AB^{1/2})^{1/2}) = 2\nor{B^{1/2}A^{1/2}}_*$. 
To conclude the proof, let $G:\{X\in\psd ~|~ \ran(B^{1/2}AB^{1/2}) \subseteq \ran(X) \} \to\R$ be the objective functional in \cref{eq:equivalent-surrogate-problem-change-of-variables} such that $G(X) = \tr(X^\dagger B^{1/2}AB^{1/2}) + \tr(X)$. Let now $X_*\in\psd$ be a minimzer for $G$ and $(X_n)_{n\in\N}$ be a minimizing sequence with $\ran(X_n)=\ran(B)$ for each $n\in\N$ and $X_n\to X_*$. Let $(\theta_n)_{n\in\N}$ such that $\theta_n = B^{\dagger/2}XB^{\dagger/2}$ for any $n\in\N$. Then we have $\theta_n\to B^{\dagger/2}X_* B^{\dagger/2}$ and by continuity $F(B^{\dagger/2}X_* B^{\dagger/2}) = G(X_*)$, hence $\min_XG(X)\leq\min_\theta F(\theta)$. 
Note that $B^{\dagger/2}X_* B^{\dagger/2}$ is a minimizer for $F$, since $F$ and $G$ have same minimum value. To see this it is sufficient to show that, given a minimizing sequence $(\theta_n)_{n\in\N}$ such that $\ran(\theta_n) = \ran(B)$ for any $n\in\N$ and $\theta_n\to\theta_*$, we have $X_n = B^{1/2}\theta_n B^{1/2}\to B^{1/2}\theta_n B^{1/2}$ and thus $F(\theta_*) = G(B^{1/2}\theta_n B^{1/2})$. We have shown that $\min_{\theta} F(\theta)\geq\min_{X}G(X)$. Therefore $\theta_* = B^{\dagger/2}X_*B^{\dagger/2} = B^{\dagger/2}(B^{1/2}AB^{1/2})^{1/2}B^{\dagger/2}$ is a minimizer of \cref{eq:original-surrogate-problem} as desired. The uniqueness of $\theta_*$ follows from the uniqueness of $X_*$ from the standard results on the variational form of the trace norm \cite{micchelli2013regularizers}. 
\end{proof}

We now have all the ingredients necessary to prove 
\cref{oracle_function}.

\OracleFunction*

\begin{proof}
We aim to minimize 
\begin{equation}\label{eq:conditional-surrogate-problem}
\min_{\substack{\tau:\Ss\to\Theta\\ \ran(W(s)) \subseteq \ran(\tau(s))} }~\Exp_{s \sim \ms}~\varphi(s,\tau(s)) \qquad \textrm{with} \qquad
\quad \varphi(s,\theta) = \frac{\tr \big( \theta^\dagger W(s) \big)}{2} 
+ \frac{2 L^2 \tr \big(\theta C(s) \big)}{n}.
\end{equation}
over the set of all measurable functions $\tau:\Ss\to\Theta$. Note that from \cref{cor:conditional-covariance-ranges}, for any $s\in\Ss$ we have $\ran(W(s))\subset\ran(C(s))$. Therefore we can apply \cref{prop:generalization-variational-trace-norm} to have that for any $s\in\Ss$, the problem
\begin{equation}
\min_{\theta \in\psd,~\ran(W(s)) \subseteq \ran(\theta)}
~ \varphi(s,\theta)
\end{equation} 
has solution 
\begin{equation}
\tau_\env(s) = \frac{\sqrt{n}}{2L}~C(s)^{\dagger/2} (C(s)^{1/2} W(s) C^{1/2})^{1/2}C(s)^{\dagger/2}.
\end{equation} 
Therefore, for any $\tau:\Ss\to\Theta$ we have
\begin{equation}
    \Exp_{s\sim\ms} ~\varphi(\tau_\env(s),s) \leq \Exp_{s\sim\ms} ~\varphi(\tau(s),s),
\end{equation}
and therefore $\Exp_{s\sim\ms} ~\varphi(\tau_\env(s),s)\leq \min_{\tau} \Exp_{s\sim\ms}~ \varphi(\tau)$. To conclude the proof we need to show that $\tau_\env$ is measurable. This follows immediately by applying Aumann’s measurable selection principle, see for instance the formulation in \citep[Lemma A.3.18]{steinwart2008support}. Under the notation of \cite{steinwart2008support}, we can apply the result by taking $h(s,\theta) = (\theta\theta^\dagger - I)W(s)$, the set $A = \{0\}\subset Y = \psd$. This guarantees the existence of a measurable function $\tau_0:\Ss\to\Theta$ such that it minimizes pointwise $\varphi(s,\cdot)$ for any $s\in\Ss$ on the set $\{\theta\in\psd ~|~ \ran(W(s))\subset\ran(\theta)\}$. The uniqueness of $\tau_\env(s)$ for each $s\in\Ss$ guarantees that $\tau_\rho = \tau_0$ is measurable as desired. 
\end{proof}

%-----------------------------------------------------------------------------------------------------------------------------------

\section{Proof of \cref{clusters_ex}}
\label{example_proof}

In this section we report the proof of the statement in \cref{clusters_ex}. 

\ClustersExample*

\begin{proof}
According to the setting described in the example, we can 
rewrite the following:
\begin{equation} \label{pp}
\begin{split}
\Exp_{s \sim \env_\Ss} \big \| C(s)^{1/2} W(s)^{1/2} \big \|_*
& = \Exp_{s \sim \env_\Ss} \big \| C^{1/2} W(s)^{1/2} \big \|_* \\
& = \Exp_{s \sim \env_\Ss} \tr\Big( \big( C^{1/2} W(s) C^{1/2} \big)^{1/2} \Big) \\
& = \frac{1}{m} \sum_{i = 1}^m \Exp_{s \sim \env_\Ss^{(i)}} 
\tr\Big( \big( C^{1/2} W(s) C^{1/2} \big)^{1/2} \Big) \\
& = \frac{1}{m} \sum_{i = 1}^m \tr\Big( \big( C^{1/2} W(a_i) C^{1/2} \big)^{1/2} \Big) \\
& = \frac{1}{m} \tr\Bigg( \sum_{i = 1}^m \big( C^{1/2} W(a_i) C^{1/2} \big)^{1/2} \Bigg) \\
& = \frac{1}{m} \tr\Bigg( \Big( \sum_{i = 1}^m C^{1/2} W(a_i) C^{1/2} \Big)^{1/2} \Bigg),
\end{split}
\end{equation}
where, in the first equality we have exploited the fact that $C(s)$ is a 
constant matrix C, in the second equality we have applied the definition
of the rewriting of the trace norm of a matrix $A$ as $\| A \|_* = \tr \big( 
(A A \trans)^{1/2} \big)$, in the third and fourth equality we have exploited 
the assumption on $\ms$, and finally, in the last equality, by point $2)$,
we managed to apply the fact that, for two matrices $A,B \in \psd$ such 
that $A^{1/2} B^{1/2} = B^{1/2}A^{1/2} = 0$, we have 
\begin{equation}
(A^{1/2} + B^{1/2}) (A^{1/2} + B^{1/2}) = A + B 
\implies 
(A + B)^{1/2} = A^{1/2} + B^{1/2}.
\end{equation} 
On the other hand, we observe that we can also write the following:
\begin{equation} \label{qq}
\begin{split}
\big \| C_\env^{1/2} W_\env^{1/2} \big \|_*
& = \big \| C^{1/2} W_\env^{1/2} \big \|_* \\
& = \tr\Big( \big( C^{1/2} W_\env C^{1/2} \big)^{1/2} \Big) \\
& = \tr\Big( \big( C^{1/2} \Exp_{s \sim \ms} W(s) C^{1/2} \big)^{1/2} \Big) \\
& = \tr\Bigg( \Big( C^{1/2} \frac{1}{m} \sum_{i = 1}^m \Exp_{s \sim \env_\Ss^{(i)}} 
W(s) C^{1/2} \Big)^{1/2} \Bigg) \\
& = \frac{1}{\sqrt{m}} \tr\Bigg( \Big( \sum_{i = 1}^m C^{1/2}
W(a_i) C^{1/2} \Big)^{1/2} \Bigg) \\
& = \frac{1}{\sqrt{m}} \tr\Bigg( \Big( C^{1/2} \sum_{i = 1}^m 
W(a_i) C^{1/2} \Big)^{1/2} \Bigg),
\end{split}
\end{equation}
where, in the first equality we have exploited the fact that $C(s)$ is a 
constant matrix C, in the second equality we have applied the definition
of the rewriting of the trace norm of a matrix $A$ as $\| A \|_* = \tr \big( 
(A A \trans)^{1/2} \big)$ and in the fourth and fifth equality we have 
exploited the assumption on $\ms$. The desired statement directly 
derives from combining \cref{pp} and \cref{qq}.
\end{proof}

%------------------------------------------------------------------------------------------------------------------------------------

\section{Proofs of the statements in \cref{Conditional Representation Meta-Learning Algorithm}}
\label{proofs_Conditional Representation Meta-Learning Algorithm_sec}

In this section we report the proofs of the statements we used
in \cref{Conditional Representation Meta-Learning Algorithm} in 
order to prove the expected excess risk bound for \cref{meta_alg} 
in \cref{bound_estimated_feature}. We start from proving the matricial
rewriting of \cref{matricial_rewriting} in \cref{proof_matricial_rewriting}.
We then prove in \cref{properties_surrogate_proof} the properties of 
the surrogate functions in \cref{properties_surrogate}. Then, in 
\cref{proof_conv_rate_surr}, we prove the convergence rate of \cref{meta_alg} 
on the surrogate problem in \cref{surrogate_linear}.

%------------------------------------------------------------------------------------------------------------------------------------

\subsection{Proof of \cref{matricial_rewriting}}
\label{proof_matricial_rewriting}

We start from proving the matricial rewriting of \cref{matricial_rewriting}..

\MatricialRewriting*

\begin{proof}
We start from observing that for any $i, j = 1, \dots, d$, 
we can rewrite the following
\begin{equation} \label{rewriting1}
\begin{split}
\Big( \big( M \Phi(s) \big) \trans M \Phi(s) \Big)_{i,j} 
& = \big \langle \big( M \Phi(s) \big) \trans (i,:), \big( M \Phi(s) \big) (:,j) \big \rangle \\
& = \big \langle \big( M \Phi(s) \big) (:,i), \big( M \Phi(s) \big) (:,j) \big \rangle \\
& = \sum_{q = 1}^m \big( M \Phi(s) \big) (:,i)_q \big( M \Phi(s) \big) (:,j)_q \\
& = \sum_{q = 1}^m \Bigg( \sum_{h = 1}^k M_{q,i,h} \Phi(s)_h \Bigg) 
\Bigg( \sum_{z = 1}^k M_{q,j,z} \Phi(s)_z \Bigg) \\
& = \sum_{q = 1}^m \sum_{h = 1}^k \sum_{z = 1}^k
M_{q,i,h} M_{q,j,z} \Phi(s)_h \Phi(s)_z \\
& = \sum_{h = 1}^k \sum_{z = 1}^k \Phi(s)_h \Phi(s)_z 
\sum_{q = 1}^m M_{q,i,h} M_{q,j,z} \\
& = \sum_{h = 1}^k \sum_{z = 1}^k \Phi(s)_h \Phi(s)_z 
\Bigg( \sum_{q = 1}^m M_{q,i,h} M_{q,j,z} \Bigg) \\
& = \sum_{h = 1}^k \sum_{z = 1}^k \Phi(s)_h \Phi(s)_z 
\big \langle M(:,i,h), M(:,j,z) \big \rangle.
\end{split}
\end{equation}
We now observe that for any $i, j = 1, \dots, d$, 
we can rewrite the following
\begin{equation} \label{rewriting2}
\begin{split}
\Big( \big( I_d \otimes \Phi(s) \trans \big) H_M \big( I_d \otimes \Phi(s) \big) \Big)_{i,j} 
& = \big \langle \big( I_d \otimes \Phi(s) \trans \big) (i,:), \big( H_M \big( I_d \otimes \Phi(s) \big) 
\big) (:,j) \big \rangle \\
& = \big \langle \big( I_d \otimes \Phi(s) \big) (:,i), \big( H_M \big( I_d \otimes \Phi(s) \big) 
\big) (:,j) \big \rangle \\
& = \sum_{n = 1}^{kd} \big( I_d \otimes \Phi(s) \big)_{n,i} 
\big( H_M \big( I_d \otimes \Phi(s) \big) \big)_{n,j} \\
& = \sum_{n = 1}^{kd} \big( I_d \otimes \Phi(s) \big)_{n,i} 
\big \langle H_M(n,:), \big( I_d \otimes \Phi(s) \big) (:,j) \big \rangle \\
& = \sum_{n = 1}^{kd} \big( I_d \otimes \Phi(s) \big)_{n,i} 
\sum_{p = 1}^{kd} \big(H_M \big)_{n,p} \big( I_d \otimes \Phi(s) \big)_{p,j} \\
& = \sum_{n = 1}^{kd} \sum_{p = 1}^{kd} \big( I_d \otimes \Phi(s) \big)_{n,i} 
\big(H_M \big)_{n,p} \big( I_d \otimes \Phi(s) \big)_{p,j} \\
& = \sum_{n = 1}^{kd} \sum_{p = 1}^{kd} \Phi(s)_h ~ \delta_{n, (i-1)k+h} 
\big(H_M \big)_{n,p} \Phi(s)_{z} ~ \delta_{p, (j-1)k+z} \\
& = \sum_{h= 1}^{k} \sum_{z = 1}^{k} \Phi(s)_h \Phi(s)_z \big( H_M \big)_{(i-1)k+h, (j-1)k+z},
\end{split}
\end{equation}
where, in the seventh equality we have exploited the fact that, by definition,
\begin{equation}
\big( I_d \otimes \Phi(s) \big)_{n,i} =
\begin{cases}
\Phi(s)_r & \text{if $r = n - (i-1)k$} \\
0 & \text{otherwise} 
\end{cases}
= \Phi(s)_r ~ \delta_{n, r + (i-1)k}.
\end{equation}
and in the last equality we have defined the new indexes
$h, z = 1, \dots, k$ as
\begin{equation}
h = n - (i-1)k
\quad \quad \quad 
z = p - (j-1)k
\end{equation}
and, as consequence, we have rewritten 
\begin{equation}
n = (i-1)k+h
\quad \quad \quad 
p = (j-1)k+z.
\end{equation}
As, a consequence, if we define $H_M$ as the matrix in 
$\Real^{dk \times dk}$ with entries
\begin{equation}
\big( H_M \big)_{(i-1)k + h, (j-1)k +z} = \big \langle 
M(:,i,h), M(:,j,z) \big \rangle, %_{\Real^m},
\end{equation}
with $i, j = 1, \dots, d$ and $h, z = 1, \dots, k$, then,
\cref{rewriting1}:
\begin{equation}
\Big( \big( M \Phi(s) \big) \trans M \Phi(s) \Big)_{i,j} 
= \sum_{h = 1}^k \sum_{z = 1}^k \Phi(s)_h \Phi(s)_z 
\big \langle M(:,i,h), M(:,j,z) \big \rangle %_{\Real^m}
\end{equation}
and \cref{rewriting2}:
\begin{equation} 
\Big( \big( I_d \otimes \Phi(s) \trans \big) H_M \big( I_d \otimes \Phi(s) \big) \Big)_{i,j} 
= \sum_{h= 1}^{k} \sum_{z = 1}^{k} \Phi(s)_h \Phi(s)_z \big( H_M \big)_{(i-1)k+h, (j-1)k+z}
\end{equation}
coincide. This coincides with the first desired statement. 
In order to prove the statement $H_M \in \mathbb{S}_+^{dk}$, 
we show that $H_M = A_M \trans A_M$, where $A_M$ is the 
matrix in $\Real^{m \times dk}$ defined as
\begin{equation}
A_M (:,(i-1)k + h) = M(:,i,h).
\end{equation}
We start from recalling that, by definition of $H_M$, we have
\begin{equation}
\big( H_M \big)_{(i-1)k + h, (j-1)k +z} = \big \langle 
M(:,i,h), M(:,j,z) \big \rangle. %_{\Real^m}.
\end{equation}
Moreover, we observe that, for any $p,q = 1,\dots,kd$,
\begin{equation}
\big( A_M \trans A_M \big)_{p,q} 
= \big \langle \big( A_M \trans \big) (p,:), A_M(:,q) \big \rangle_{\Real^m}
= \big \langle A_M(:,p), A_M(:,q) \big \rangle. %_{\Real^m}.
\end{equation}
As a consequence, the desired statement is satisfied if we define
\begin{equation}
\big( A_M \big) (:,(i-1)k + h) = M(:,i,h).
\end{equation}
We now prove the last statement.
Let $(e_i)_{i = 1}^d$ be the canonical basis in $\Real^d$.
By the definition of the trace and the rewriting of $\tau(s)$ in 
\cref{matricial_rewriting}, denoting by ${\rm vec}$ the vectorization 
operation, we can rewrite
\begin{equation} \label{pluto}
\begin{split}
\tr \big(\tau(s)\big) 
& = \sum_{i = 1}^d \big \langle e_i, \tau(s) e_i \big \rangle \\ %_{\Real^d} \\
& = \sum_{i = 1}^d \big \langle e_i, \big( I_d \otimes \Phi(s) \trans \big) 
H_M \big( I_d \otimes \Phi(s) \big) e_i \big \rangle \\ %_{\Real^d} \\
& = \sum_{i = 1}^d e_i \trans \big( I_d \otimes \Phi(s) \trans \big) 
H_M \big( I_d \otimes \Phi(s) \big) e_i \\
& = \sum_{i = 1}^d \Big( \big( I_d \otimes \Phi(s) \big) e_i \Big) \trans 
H_M \big( I_d \otimes \Phi(s) \big) e_i \\
& = \sum_{i = 1}^d \Big( {\rm vec} \big( \Phi(s) e_i \trans \big) \Big) \trans 
H_M {\rm vec} \big( \Phi(s) e_i \trans \big) \\
& =  \tr \Big( H_M \sum_{i = 1}^d {\rm vec} \big( \Phi(s) e_i \trans \big)
{\rm vec} \big( \Phi(s) e_i \trans \big) \trans \Big) \\
& \le \tr \big( H_M \big) \Bigg \| \sum_{i = 1}^d {\rm vec} \big( \Phi(s) e_i \trans \big)
{\rm vec} \big( \Phi(s) e_i \trans \big) \trans \Bigg \|_\infty \\
& = \tr \big( H_M \big) \big \| \Phi(s) \big \|_{\Real^k}^2,
\end{split}
\end{equation}
where, in the fifth equality, we have applied the relation
\begin{equation}
\big( C \trans \otimes A \big) {\rm{vec}}(B) = 
{\rm{vec}} (ABC)
\end{equation}
with $A = \Phi(s)$, $B = e_i \trans$ and $C = I_d$, i.e.
\begin{equation}
\big( I_d \otimes \Phi(s) \big) e_i = {\rm vec} \big( \Phi(s) e_i \trans \big), 
\end{equation}
in the inequality we have applied Holder's inequality
and in the last equality we have applied the following 
proposition.
\end{proof}

\begin{proposition}
For any $i = 1,\dots,d$, define
\begin{equation}
v_i = {\rm vec} \big( \Phi(s) e_i \trans \big)
\end{equation}
Then, 
\begin{equation}
\Bigg \| \sum_{i = 1}^d {\rm vec} \big( \Phi(s) e_i \trans \big)
{\rm vec} \big( \Phi(s) e_i \trans \big) \trans \Bigg \|_\infty 
= \Bigg \| \sum_{i = 1}^d v_i v_i \trans \Bigg \|_\infty 
= \big \| \Phi(s) \big \|^2.
\end{equation}
\end{proposition}

\begin{proof}
We start from observing that, for any $i,j = 1,\dots,d$, we have
\begin{equation}
\begin{split}
v_i \trans v_j 
& = {\rm vec} \big( \Phi(s) e_i \trans \big) \trans {\rm vec} \big( \Phi(s) e_j \trans \big) \\
& = \tr \big( e_i \Phi(s) \trans \Phi(s) e_j \trans \big) \\
& = \tr \big( \Phi(s) \trans \Phi(s) e_j \trans e_i \big) \\
& = \Phi(s) \trans \Phi(s) e_j \trans e_i \\
& = \big \| \Phi(s) \big \|^2 \delta_{i,j},
\end{split}
\end{equation}
where, in the second equality, we have used 
the property of the operator $\rm{vec}$:
\begin{equation}
{\rm vec} (A) \trans {\rm vec} (B) 
= \tr \big( A \trans B \big)
\end{equation}
with
\begin{equation}
A = \Phi(s) e_i \trans
\quad \quad
B = \Phi(s) e_j \trans.
\end{equation}
As a consequence, the vectors
\begin{equation}
\tilde{v_i} = \frac{v_i}{\| v_i \|} = \frac{v_i}{\big \| \Phi(s) \big \|}
\quad \quad
i = 1, \dots, d,
\end{equation}
form an orthonormal basis of the space. 
Moreover, we can rewrite the operator above as
follows
\begin{equation}
\sum_{i = 1}^d {\rm vec} \big( \Phi(s) e_i \trans \big)
{\rm vec} \big( \Phi(s) e_i \trans \big) \trans
= \sum_{i = 1}^d v_i v_i \trans
= \sum_{i = 1}^d \big \| \Phi(s) \big \|^2 \tilde{v_i} \tilde{v_i} \trans.
\end{equation}
The rewriting above coincides with the eigenvalues' decomposition
of the operator: the vectors $\tilde{v_i}$ are the eigenvectors with
associated constant eigenvalues $\big \| \Phi(s) \big \|^2$. As a 
consequence, we can conclude that 
\begin{equation}
\Bigg \| \sum_{i = 1}^d {\rm vec} \big( \Phi(s) e_i \trans \big)
{\rm vec} \big( \Phi(s) e_i \trans \big) \trans \Bigg \|_\infty 
= \big \| \Phi(s) \big \|^2.
\end{equation}
\end{proof}

%-------------------------------------------------------------------------------------------------------------------------------------------

\subsection{Proof of \cref{properties_surrogate}}
\label{properties_surrogate_proof}

We now prove the properties of the surrogate functions in 
\cref{properties_surrogate}.

\PropertiesSurrogate*

\begin{proof} 
We are interested in studying the properties of the surrogate function 
$\LL \big (\cdot, \cdot, s, \Zn \big ): \mathbb{S}_+^{dk} \times \psd \to \Real$
in \cref{surrogate_linear}. We start from observing that, 
such a function coincides with the composition of the function
\begin{equation} \label{banana5}
\begin{split}
& \theta \in \psd \mapsto \Delta(\theta, \Zn) = F(\theta,\Zn) + G(\theta, \Zn) \in \Real \\
& F(\theta,\Zn) = \min_{w \in \Real^d} ~ \cR_{\Zn,\theta}(w) 
\quad \quad 
\cR_{\Zn,\theta}(w) = \frac{1}{n} \sum_{i = 1}^n \ell(\langle x_i, w \rangle, y_i) 
+ \frac{\la}{2} \big \langle w, \theta^\dagger w \big \rangle + \iota_{\ran(\theta)}(w) \\
& G(\theta, \Zn) = \frac{2 L^2}{n}  \tr \Big(\theta \frac{X \trans X}{n} \Big).
\end{split}
\end{equation}
with the linear transformation 
\begin{equation}
s \in \Ss \mapsto \tau_{H,C}(s) =  
\big( I_d \otimes \Phi(s) \trans \big) H \big( I_d \otimes \Phi(s) \big) + C \in \psd.
\end{equation}
In other words, for any $H \in \mathbb{S}_+^{dk}$ and $C \in \psd$,
we can write
\begin{equation}
\LL \big (H, C, s, \Zn \big ) = \Delta(\tau_{H,C}(s), \Zn)
= F(\tau_{H,C}(s),\Zn) + G(\tau_{H,C}(s), \Zn).
\end{equation}
We now observe that both the functions $F(\cdot, \Zn)$ and $G(\cdot, \Zn)$ 
are both convex ($F(\cdot, \Zn)$ is convex since it is the minimum of a jointly 
convex function see \cite{denevi2019online} and $G(\cdot, \Zn)$ is a linear function).
As a consequence, the function $\Delta(\cdot, \Zn)$ is convex over $\psd$.
This implies the convexity of the surrogate function $\LL \big (\cdot, 
\cdot, s, \Zn \big )$ over $\mathbb{S}_+^{dk} \times \psd$ (composition of a 
convex function with a linear transformation). In order to get the closed form 
of the gradient in \cref{gradient1} we proceed in a similar way as in \cite{denevi2020advantage}.
More precisely, we start from recalling that, as already observed in \cite{denevi2019online},
thanks to strong duality in the within-task problem, for any $\theta \in \psd$, we 
can rewrite
\begin{equation}
F(\theta,\Zn) = \min_{w \in \ran(\theta)} \cR_{\Zn, \theta}(w) 
 = \max_{\alpha \in \Real^n} \Big\{ - \frac{1}{n} \sum_{i = 1}^n \ell_i^*(\alpha_i) - 
\frac{1}{2 n^2} \tr \big( \theta X \trans \alpha \alpha \trans X \big) \Big \},
\end{equation}
where, $\ell_i^*(\cdot)$ denotes the Fenchel conjugate of $\ell_i(\cdot) = \ell(\cdot,y_i)$ and 
$\alpha \in \Real^n$ coicides with the dual variable. As a consequence,
we can rewrite 
\begin{equation}
\begin{split}
\Delta(\theta, \Zn) & = F(\theta,\Zn) + G(\theta, \Zn) \\
& = \max_{\alpha \in \Real^n} \Big\{ - \frac{1}{n} \sum_{i = 1}^n \ell_i^*(\alpha_i) - 
\frac{1}{2 n^2} \tr \big( \theta X \trans \alpha \alpha \trans X \big) \Big \} 
+ \frac{2 L^2}{n}  \tr \Big(\theta \frac{X \trans X}{n} \Big) \\
& = \max_{\alpha \in \Real^n} \Bigg\{ - \frac{1}{n} \sum_{i = 1}^n \ell_i^*(\alpha_i) 
+ \tr \Bigg( \theta \Big(- \frac{X \trans \alpha \alpha \trans X}{2 n^2} 
+ \frac{2 L^2 X \trans X}{n^2} \Big) \Bigg) \Bigg \}.
\end{split}
\end{equation}
As a consequence, we have
\begin{equation}
\begin{split}
\Delta(\tau_{H,C}, \Zn) 
= \max_{\alpha \in \Real^n} \Bigg\{& - \frac{1}{n} \sum_{i = 1}^n \ell_i^*(\alpha_i) 
+ \tr \Bigg( \big( I_d \otimes \Phi(s) \trans \big) H \big( I_d \otimes \Phi(s) \big) 
\Big(- \frac{X \trans \alpha \alpha \trans X}{2 n^2} + \frac{2 L^2 X \trans X}{n^2} \Big) \Bigg) \\
& + \tr \Bigg( C \Big(- \frac{X \trans \alpha \alpha \trans X}{2 n^2} + \frac{2 L^2 X \trans X}{n^2} 
\Big) \Bigg)\Bigg \} \\
= \max_{\alpha \in \Real^n} \Bigg\{& - \frac{1}{n} \sum_{i = 1}^n \ell_i^*(\alpha_i) 
+ \tr \Bigg( H \big( I_d \otimes \Phi(s) \big) 
\Big(- \frac{X \trans \alpha \alpha \trans X}{2 n^2} + \frac{2 L^2 X \trans X}{n^2} \Big) 
\big( I_d \otimes \Phi(s) \trans \big) \Bigg) \\
& + \tr \Bigg( C \Big(- \frac{X \trans \alpha \alpha \trans X}{2 n^2} + \frac{2 L^2 X \trans X}{n^2} 
\Big) \Bigg)\Bigg \} \\
= \max_{\alpha \in \Real^n}  ~ & Q(\alpha, H, C, s, \Zn),
\end{split}
\end{equation}
where we have introduced the function 
\begin{equation}
\begin{split}
 Q(\alpha, H, C, s, \Zn) = &- \frac{1}{n} \sum_{i = 1}^n \ell_i^*(\alpha_i) 
+ \tr \Bigg( H \big( I_d \otimes \Phi(s) \big) \Big(- \frac{X \trans \alpha \alpha \trans X}{2 n^2} 
+ \frac{2 L^2 X \trans X}{n^2} \Big) \big( I_d \otimes \Phi(s) \trans \big) \Bigg) \\
& + \tr \Bigg( C \Big(- \frac{X \trans \alpha \alpha \trans X}{2 n^2} + \frac{2 L^2 X \trans X}{n^2} 
\Big) \Bigg).
\end{split}
\end{equation}
Hence, applying \citep[Lemma 44]{denevi2019online}, we know that, once computed
a maximizer $\alpha_{\tau_{H,C}(s)}$ of the function above $\alpha 
\in \Real^n \mapsto Q(\alpha,H,C,s,\Zn)$, 
\begin{equation}
\nabla Q(\alpha_{\tau_{H,C}(s)}, \cdot,\cdot, s,\Zn)(H,C) \in 
\frac{\partial \Delta(\tau_{H,C}(s),\Zn)}{\partial (H,C)}
= \frac{\partial \LL\big (H, C, s, \Zn \big)}{\partial(H,C)}.
\end{equation}
As a consequence, since for a given matrix $A$, 
$\nabla \tr \big( \cdot A)(H) = A$, we get that 
\begin{equation}
\nabla \LL\big (\cdot, \cdot, s, \Zn \big )(H, C) 
= \Big( \big( I_d \otimes \Phi(s) \big) \hat \nabla \big( I_d \otimes \Phi(s) \trans \big),
\hat \nabla \Big) \in \frac{\partial \LL\big (H, C, s, \Zn \big)}{\partial(H,C)},
\end{equation}
with
\begin{equation} \label{gigia}
\hat \nabla = - \frac{X \trans \alpha_{\tau_{H,C}(s)} \alpha_{\tau_{H,C}(s)} \trans X}{2 n^2} 
+ \frac{2 L^2 X \trans X}{n^2}.
\end{equation}
Finally, in order to get the desired closed form in \cref{gradient1}, 
we just need to observe that, according to the optimality conditions
of the within-task problem in (see \citep[Lemma 44]{denevi2019online}) 
with $\theta \in \psd$, we have that 
\begin{equation}
X \trans \alpha_\theta = - n \theta^\dagger w_\theta.
\end{equation}
As a consequence, we can rewrite \cref{gigia} as follows
by using the primal solution of the within-task problem:
\begin{equation} 
\hat \nabla = - \frac{\la}{2} \tau_{H,C}(s)^\dagger w_{\tau_{H,C}(s)}
w_{\tau_{H,C}(s)} \trans \tau_{H,C}(s)^\dagger + \frac{2 L^2 X \trans X}{n^2}.
\end{equation}
Finally, we observe that, by the closed form in \cref{gradient1},
\begin{equation} \label{decomposition_gradient}
\big \| \nabla \LL\big (\cdot, \cdot,s, \Zn \big )(H,C)  \big \|_F 
\le \big \| \nabla \LL\big (\cdot, \cdot,s, \Zn \big )(H,C)  \big \|_*
\le \text{A} + \text{B} + \text{C} + \text{D}
\end{equation}
with 
\begin{equation}
\begin{split}
\text{A} & = \Bigg \| \big( I_d \otimes \Phi(s) \big) 
\frac{X \trans \alpha_{\tau_{H,C}(s)} \alpha_{\tau_{H,C}(s)} 
\trans X}{2 n^2} \big( I_d \otimes \Phi(s) \trans \big) \Bigg \|_* \\
\text{B} & = \Bigg \| \big( I_d \otimes \Phi(s) \big) 
\frac{2 L^2 X \trans X}{n^2} \big( I_d \otimes \Phi(s) \trans \big) \Bigg \|_* \\
\text{C} & = \Bigg \| \frac{X \trans \alpha_{\tau_{H,C}(s)} \alpha_{\tau_{H,C}(s)} 
\trans X}{2 n^2} \Bigg \|_* \\
\text{D} & = \Bigg \| \frac{2 L^2 X \trans X}{n^2} \Bigg \|_*.
\end{split}
\end{equation}
We now observe that all the matrices inside the trace norms above 
are positive semidefinite (as a matter of fact, if a matrix $Q$ is 
positive semidefinite, then, $P \trans Q P$ is positive semidefinite
for any matrix $P$). As a consequence, all the trace norms above 
coincide with the trace of the corresponding matrices, namely,
\begin{equation}
\begin{split}
\text{A} & = \tr \Bigg( \big( I_d \otimes \Phi(s) \big) 
\frac{X \trans \alpha_{\tau_{H,C}(s)} \alpha_{\tau_{H,C}(s)} 
\trans X}{2 n^2} \big( I_d \otimes \Phi(s) \trans \big) \Bigg) \\
\text{B} & = \tr \Bigg( \big( I_d \otimes \Phi(s) \big) 
\frac{2 L^2 X \trans X}{n^2} \big( I_d \otimes \Phi(s) \trans \big) \Bigg) \\
\text{C} & = \tr \Bigg( \frac{X \trans \alpha_{\tau_{H,C}(s)} \alpha_{\tau_{H,C}(s)} 
\trans X}{2 n^2} \Bigg) \\
\text{D} & = \tr \Bigg( \frac{2 L^2 X \trans X}{n^2} \Bigg).
\end{split}
\end{equation}
We now observe that, proceeding as above in 
\cref{pluto} and exploiting \cref{ass_3}, we can write 
\begin{equation}
\begin{split}
\text{A} & \le \big \| \Phi(s) \big \|^2 \tr \Bigg( 
\frac{X \trans \alpha_{\tau_{H,C}(s)} \alpha_{\tau_{H,C}(s)} 
\trans X}{2 n^2}  \Bigg) 
= \big \| \Phi(s) \big \|^2 \text{C} 
\le K^2 \text{C} \\
\text{B} & \le \big \| \Phi(s) \big \|^2 \tr \Bigg( \frac{2 L^2 X \trans X}{n^2} \Bigg)
= \big \| \Phi(s) \big \|^2 \text{D}
\le K^2 \text{D}.
\end{split}
\end{equation}
Hence, combining everything in \cref{decomposition_gradient},
we get 
\begin{equation} \label{decomposition_gradient_2}
\big \| \nabla \LL\big (\cdot, \cdot, s, \Zn \big )(H,C)  \big \|_F 
\le (1 + K^2) \big( \text{C} + \text{D} \big).
\end{equation}
The desired statement derives from observing that, 
since, by \cref{ass_1}, $\tr \big( X \trans \alpha_{\tau_{H,C}(s)} 
\alpha_{\tau_{H,C}(s)} \trans X \big) \le (n L R)^2$ (see 
\citep[Lemma 44]{denevi2019online}) and $\tr \big( X \trans X \big) 
= \tr \big( X X \trans \big) 
= \sum_{i = 1}^n \| x_i \|^2 \le n R^2$, then
\begin{equation}
\text{C} \le \frac{(LR)^2}{2 \la} 
\quad \quad \quad 
\text{D} \le \frac{2 (LR)^2}{n}.
\end{equation}
\end{proof}

%----------------------------------------------------------------------------------------------------------------------------------

\subsection{Convergence rate of \cref{meta_alg} on the surrogate 
problem in \cref{surrogate_linear}}
\label{proof_conv_rate_surr}

We now give the convergence rate of \cref{meta_alg} on the surrogate
problem in \cref{surrogate_linear}.

\begin{restatable}[Convergence rate on the surrogate problem in \cref{surrogate_linear}]{proposition}{ConvergenceSurrogate} \label{convergence_surrogate}
Let $\thickbar H$ and $\thickbar C$ be the average of the iterations obtained from the application of \cref{meta_alg} over the training data $(\Zn_t, s_t)_{t = 1}^T$ with constant meta-step size $\gamma > 0$. Then, under \cref{ass_1} and \cref{ass_3}, for any $\tau_{H,C} \in \T_\Phi$, in expectation w.r.t. the sampling of $(\Zn_t, s_t)_{t = 1}^T$,
\begin{equation}
\Exp ~ \hat \ee_\env \bigl(\tau_{\thickbar H, \thickbar C} \bigr) 
- \hat \ee_\env \bigl( \tau_{H,C} \bigr) \le 
\frac{\gamma (1+ K^2)^2(LR)^4}{2 \la^2} \Big( \frac{1}{2} + \frac{2}{n}\Big)^2
+ \frac{\big\| (H - H_0, C - C_0) \big\|_F^2}{2 \gamma T}.
\end{equation}
\end{restatable}

\begin{proof}
We observe that \cref{meta_alg} coincides with projected Stochastic Gradient 
Descent applied to the convex and Lipschitz (see \cref{properties_surrogate})
surrogate problem in \cref{surrogate_linear}:
\begin{equation}
\min_{H \in \mathbb{S}^{dk}_+, C \in \psd} ~ \hat \ee_\env(\tau_{H, C})
\quad \quad 
\hat \ee_\env(\tau_{H, C}) = \Exp_{(\task, s) \sim \env} 
~ \Exp_{\Zn \sim \task^n} ~ \LL\big (H, C, s, \Zn \big ).
\end{equation}
As a consequence, by standard arguments
(see e.g. \citep[Lemma $14.1$, Thm. $14.8$]{shalev2014understanding} 
and references therein), for any $\tau_{H,C} \in \T_\Phi$, we have
\begin{equation}
\Exp~ \hat \ee_\env \bigl(\tau_{\thickbar H, \thickbar C} \bigr) 
- \hat \ee_\env \bigl( \tau_{H, C} \bigr) \le \frac{\gamma}{2 T} 
\sum_{t = 1}^T \Exp ~ \big \| \nabla \LL\big (\cdot, \cdot, 
s, \Zn_t \big )(H_t, C_t) \big\|_F^2
+ \frac{\big\| (H - H_0, C - C_0) \big\|_F^2}{2 \gamma T}.
\end{equation}
The desired statement derives from combining this 
bound with the bound on the norm of the meta-subgradients
in \cref{properties_surrogate}.
\end{proof}

%-----------------------------------------------------------------------------------------------------------------------------------

\subsection{Proof of \cref{bound_estimated_feature}}
\label{feature_proof_final_thm}

We now have all the ingredients necessary to prove \cref{bound_estimated_feature}.

\BoundEstimatedFeature*

\begin{proof}
We start from observing thta, in expectation w.r.t. the meta-training set, 
for any fixed conditioning function $\tau_{H,C} \in \T_\Phi$, we can write 
the following decomposition
\begin{equation} \label{decomposition}
\begin{split}
\Exp ~ \ee_\env(\tau_{\thickbar H,\thickbar C}) - \ee_\env^*  
& \le \Exp ~ \hat \ee_\env(\tau_{\thickbar H,\thickbar C}) - \ee_\env^*  \\
& = \Exp ~ \hat \ee_\env(\tau_{\thickbar H,\thickbar C}) - \ee_\env^*  \pm \hat \ee_\env(\tau_{H,C}) \\
& = \underbrace{\Exp ~ \hat \ee_\env(\tau_{\thickbar H,\thickbar C}) - \hat \ee_\env(\tau_{H,C})}_{\text{A}(\tau_{H,C})} + \underbrace{\hat \ee_\env(\tau_{H,C}) - \ee_\env^* }_{\text{B}(\tau_{H,C})},
\end{split}
\end{equation}
where in the inequality above we have exploited the fact that,
for any $\tau \in \T$, $\ee_\env(\tau) \le \hat \ee_\env(\tau)$
(see \cref{majorization}). We now observe that the term $\text{A}(\tau_{H,C})$
can be controlled according to the convergence properties of the 
meta-algorithm in \cref{meta_alg} as described in \cref{convergence_surrogate}:
\begin{equation}
\Exp ~ \hat \ee_\env \bigl(\tau_{\thickbar H, \thickbar C} \bigr) 
- \hat \ee_\env \bigl( \tau_{H,C} \bigr) \le 
\frac{\gamma (1+ K^2)^2(LR)^4}{2} \Big( \frac{1}{2} + \frac{2}{n}\Big)^2
+ \frac{\big\| (H-H_0, C-C_0) \big\|_F^2}{2 \gamma T}.
\end{equation}
Regarding the term $\text{B}(\tau_{H,C})$, we observe that, for any $\tau$, we 
can rewrite 
\begin{equation}
\begin{split}
\text{B}(\tau) 
& = \hat{\ee}_\env(\tau) - \ee_\env^*  \\ 
& = \Exp_{(\task,s) \sim \env} ~ \Exp_{\Zn \sim \task^n} ~ \Big[ \cR_{\Zn,\tau(s)}(A(\tau(s), \Zn))
- \cR_\task(w_\task) \Big] + \frac{2 L^2  \Exp_{(\task,s) \sim \env} ~ \tr \big(\tau(s) \Exp_{x \sim \eta_\task} x x \trans \big)}{n} \\
& \le \frac{\Exp_{(\task,s) \sim \env} ~ \tr \big( \tau(s)^\dagger w_\task w_\task \trans \big)}{2} 
 + \frac{2 L^2  \Exp_{(\task,s) \sim \env} ~ \tr \big(\tau(s) \Exp_{x \sim \eta_\task} x x \trans \big)}{n},
\end{split}
\end{equation}
where in the inequality we have exploited the fact that, thanks to the definition of the 
algorithm, for any $(\task,s) \sim \env$, we can write
\begin{equation}
\Exp_{\Zn \sim \task^n} ~ \Big[ \cR_{\Zn,\tau(s)}(A(\tau(s), \Zn))
- \cR_\task(w_\task) \Big]  \le
\frac{\tr \big( \tau(s)^\dagger w_\task w_\task \trans \big)}{2}.
\end{equation}
Combining the bounds on the two terms above in 
\cref{decomposition}, we get
\begin{equation}
\begin{split}
\Exp ~ \ee_\env(\tau_{\thickbar H,\thickbar C}) - \ee_\env^*  
\le & \frac{\Exp_{(\task,s) \sim \env} ~ \tr \big( \tau_{H,C}(s)^\dagger w_\task w_\task \trans \big)}{2} 
+ \frac{2 L^2  \Exp_{(\task,s) \sim \env} ~ \tr \big(\tau_{H,C}(s) \Exp_{x \sim \eta_\task} x x \trans \big)}{n} \\
& + \frac{\gamma (1+ K^2)^2(LR)^4}{2} \Big( \frac{1}{2} + \frac{2}{n}\Big)^2
+ \frac{\big\| (H-H_0, C-C_0) \big\|_F^2}{2 \gamma T}.
\end{split}
\end{equation}
The desired statement derives from optimizing w.r.t. the hyper-parameter
$\gamma > 0$.
\end{proof}

%-------------------------------------------------------------------------------------------------------------------------------------------

\section{Experimental Details}
\label{experimental_details} 

In this section we report the experimental details we missed in the main body.
Specifically, we report the details regarding the tuning of the hyper-parameter 
$\gamma$ and the characteristics of the machine we used for running our experiments.
%and the complexity of our method in \cref{meta_alg}.

\paragraph{Synthetic Clusters}
In order to tune the hyper-parameter $\gamma$ we applied the procedure above 
with $14$ candidates values for $\gamma$ in the range $[10^{-5}, 10^5]$ with logarithmic 
spacing and we evaluated the performance of the estimated meta-parameters 
(linear representations) by using $T =T_{\rm tr} = 500$, $T_{\rm va} = 300$, $T_{\rm te} = 100$ 
of the available tasks for meta-training, meta-validation and meta-testing, respectively. In order 
to train and to test the inner algorithm, we splitted each within-task dataset into $n = n_{\rm tr} = 
50\% ~ n_{\rm tot}$ for training and $n_{\rm te} = 50\% ~ n_{\rm tot}$ for test. 

\paragraph{Lenk Dataset}
In order to tune the hyper-parameter $\gamma$ we applied the procedure above 
with $14$ candidates values for $\gamma$ in the range $[10^{-5}, 10^5]$ with logarithmic 
spacing and we evaluated the performance of the estimated meta-parameters 
(linear representations) by using $T =T_{\rm tr} = 100$, $T_{\rm va} = 40$, $T_{\rm te} = 30$ 
of the available tasks for meta-training, meta-validation and meta-testing, respectively. In order 
to train and to test the inner algorithm, we splitted each within-task dataset into $n = n_{\rm tr} 
= 16$ for training and $n_{\rm te} = 4$ for test. 

\paragraph{Movieles-100k Dataset}
In order to tune the hyper-parameter $\gamma$ we applied the procedure above 
with $14$ candidates values for $\gamma$ in the range $[10^{-5}, 10^5]$ with logarithmic 
spacing and we evaluated the performance of the estimated meta-parameters 
(linear representations) by using $T =T_{\rm tr} = 200$, $T_{\rm va} = 100$, $T_{\rm te} = 100$ 
of the available tasks for meta-training, meta-validation and meta-testing, respectively. In order 
to train and to test the inner algorithm, we splitted each within-task dataset into $n = n_{\rm tr} 
= 15$ for training and $n_{\rm te} = 5$ for test. 

\paragraph{Jester-1 Dataset}
In order to tune the hyper-parameter $\gamma$ we applied the procedure above 
with $14$ candidates values for $\gamma$ in the range $[10^{-5}, 10^5]$ with logarithmic 
spacing and we evaluated the performance of the estimated meta-parameters 
(linear representations) by using $T =T_{\rm tr} = 250$, $T_{\rm va} = 100$, $T_{\rm te} = 100$ 
of the available tasks for meta-training, meta-validation and meta-testing, respectively. In order 
to train and to test the inner algorithm, we splitted each within-task dataset into $n = n_{\rm tr} 
= 15$ for training and $n_{\rm te} = 5$ for test. \\

All the experiments were conducted on a workstation with 4 Intel Xeon 
E5-2697 V3 2.60Ghz CPUs and 256GB RAM.

%The variant of our method in \cref{meta_alg} for biased regularization using the batch 
%inner algorithm in \cref{RERM_feature} has a time and space complexity $\mathcal{O}(d (k+n))$. 
%The variant using the online inner algorithm in \cref{online_inner_algorithm} has a time and space complexity $\mathcal{O}(dk)$.

\end{document}